\documentclass{article}

    \PassOptionsToPackage{numbers, compress}{natbib}

    \usepackage[final]{neurips_2020}

\usepackage{microtype}
\usepackage{graphicx,nicefrac}
\usepackage{booktabs} 

\usepackage{hyperref}
\hypersetup{colorlinks=true,linkcolor=blue,citecolor=blue}

\usepackage{amsmath,amssymb,amsthm}
\usepackage{booktabs} 
\usepackage{amscd}
\usepackage{mathrsfs} 
\usepackage[shortlabels]{enumitem}
\usepackage{euscript}
\usepackage{graphicx}
\usepackage{amscd,bm}

\usepackage[utf8]{inputenc} 
\usepackage[T1]{fontenc}    
\usepackage{hyperref}       
\usepackage{url}            
\usepackage{booktabs}       
\usepackage{amsfonts}       
\usepackage{nicefrac}       
\usepackage{microtype}      
\usepackage{makecell}

\DeclareMathAlphabet{\pazocal}{OMS}{zplm}{m}{n}
\DeclareMathAlphabet\mathbfcal{OMS}{cmsy}{b}{n}

\usepackage{algorithm}
\usepackage{algorithmic}
\newtheorem{theorem}{Theorem}
\newtheorem{definition}{Definition}
\newtheorem{lemma}{Lemma}
\newtheorem{proposition}{Proposition}

\providecommand{\nor}[1]{\ensuremath{\left\lVert {#1} \right\rVert}}

\providecommand{\scal}[2]{\ensuremath{\left\langle{#1},{#2}\right\rangle}}

\providecommand{\scalT}[2]{\ensuremath{\left\langle{#1},{#2}\right\rangle}}

\providecommand{\scal}[2]{\ensuremath{\left\langle{#1},{#2}\right\rangle}}

\def\bit{\begin{itemize}}
\def\eit{\end{itemize}}

\def\ben{\begin{enumerate}}
\def\een{\end{enumerate}}

\usepackage{listings}
\usepackage{color}

\definecolor{dkgreen}{rgb}{0,0.6,0}
\definecolor{gray}{rgb}{0.5,0.5,0.5}
\definecolor{mauve}{rgb}{0.58,0,0.82}

\usepackage{enumitem}

\usepackage{subcaption}

\title{Unbalanced Sobolev Descent}

\author{%
  Youssef Mroueh , Mattia Rigotti \\
  IBM Research AI \\
  \texttt{mroueh@us.ibm.com,mrg@zurich.ibm.com} \\
}

\begin{document}

\maketitle

\begin{abstract}
We introduce Unbalanced Sobolev Descent (USD), a particle descent algorithm for transporting a high dimensional source distribution to a target distribution that does not necessarily have the same mass. We define the \emph{Sobolev-Fisher} discrepancy between distributions and show that it relates to advection-reaction transport equations and the Wasserstein-Fisher-Rao metric between distributions. USD transports particles along gradient flows of the witness function of the Sobolev-Fisher discrepancy (advection step) and reweighs the mass of particles with respect to this witness function (reaction step). The reaction step can be thought of as a birth-death process of the particles with rate of growth proportional to the witness function. When the \emph{Sobolev-Fisher witness} function is estimated in a Reproducing Kernel Hilbert Space (RKHS), under mild assumptions we show that USD converges asymptotically (in the limit of infinite particles) to the target distribution in the Maximum Mean Discrepancy (MMD) sense. We then give two methods to estimate the \emph{Sobolev-Fisher witness} with neural networks, resulting in two Neural USD algorithms. The first one implements the reaction step with mirror descent on the weights, while the second implements it through a birth-death process of particles. We show on synthetic examples that USD transports distributions with or without conservation of mass faster than previous particle descent algorithms, and finally demonstrate its use for molecular biology analyses where our method is naturally suited to match developmental stages of populations of differentiating cells based on their single-cell RNA sequencing profile. Code is available at \url{http://github.com/ibm/usd}.
\end{abstract}

\section{Introduction}

Particle flows such as Stein Variational Gradient descent \citep{steindescent}, Sobolev descent \citep{SD} and MMD flows \citep{arbel2019maximum}, allow the transport of a source distribution to a target distribution, following paths that progressively decrease a discrepancy between distributions (Kernel Stein discrepancy and MMD, respectively).
Particle flows can be seen through the lens of Optimal Transport as gradient flows in the Wasserstein geometry \citep{santambrogio2017euclidean}, and they've been recently used to analyze the dynamics of gradient descent in over-parametrized neural networks in \cite{chizat2018global} and of Generative Adversarial Networks (GANs) training \citep{SD}.

Unbalanced Optimal Tansport \citep{chizat2015unbalanced, chizat2018interpolating, chizat2018scaling, sejourne2019sinkhorn} is a new twist on the classical Optimal Transport theory \citep{Villani}, where the total mass between source and target distributions may not be conserved.
The Wasserstein Fisher-Rao (WFR) distance introduced in \cite{chizat2018interpolating} gives a dynamic formulation similar to the so-called Benamou-Brenier dynamic form of the Wasserstein-$2$ distance \citep{dynamicTransport}, where the dynamics of the transport is governed by an advection term with a velocity field $V_t$ and a reaction term with a rate of growth $r_{t}$, corresponding to the construction and destruction of mass with the same rate: 
\begin{align}
 &\text{WFR}^2(p,q)= \inf_{q_t,V_t,r_t} \int_{0}^1 \int  (\nor{V_t(x)}^2 + \frac{\alpha}{2} r^2_t(x)) dq_{t}(x) dt\nonumber\\
 &\quad\text{ s.t } \frac{\partial q_t(x)}{\partial t}=-\text{div}(q_t(x) V_t(x))+\alpha~ r_{t}(x)q_{t}(x), \qquad q_0=q, q_1=p.
 \label{eq:BenamouWFR}
 \end{align}

From a particle flow point of view, this advection-reaction in Unbalanced Optimal Transport corresponds to processes of birth and death, where particles are created or killed in the transport from source to target.
Particle gradient descent using the WFR geometry have been used in the analysis of over-parameterized neural networks and implemented as Birth-Death processes in  \cite{deathbirthJoan} and as conic descent in \cite{chizat2019sparse}.
In the context of particles transportations, \cite{deathbirthLangevin} showed that birth and death processes can accelerate the Langevin diffusion. On the application side, Unbalanced Optimal Transport is a powerful tool in  biological modeling.
For instance, the trajectories of a tumor growth have been modeled in the WFR framework by \cite{chizat2017tumor}.
\cite{schiebinger2019optimal} and \cite{yang2018scalable} used Unbalanced Optimal Transport to find differentiation trajectories of cells during development.

%
The  dynamic formulation of WFR is challenging as it requires solving PDEs.
One can use the unbalanced Sinkhorn divergence and apply an Euler scheme to find the trajectories between source and target as done in \cite{feydy2018interpolating} but this does not give any convergence guarantees.

In this paper we take another approach similar to the one of Sobolev Descent \citep{SD}.
We introduce the Kernel Sobolev-Fisher discrepancy that is related to WFR and has the advantage of having a closed form solution.
We present a particle descent algorithm in the unbalanced case named \emph{Unbalanced Sobolev Descent} (USD) that consists of two steps: an advection step that uses the gradient flows of a witness function of the Sobolev-Fisher discrepancy, and  a reaction  step that reweighs the particles according to the witness function. 
We show theoretically that USD is convergent in the Maximum Mean Discrepancy sense (MMD), that the reaction step accelerates the convergence, in the sense that it results in strictly steeper descent directions, and give a variant where the witness function is efficiently estimated as a neural network.
We then empirically demonstrate the effectiveness and acceleration of USD in synthetic experiments, image color transfer tasks, and finally use it to model the developmental trajectories of populations of cells from single-cell RNA sequencing data \citep{schiebinger2019optimal}.

\section{Sobolev-Fisher Discrepancy}

In this Section we define the Sobolev-Fisher Discrepancy ($\text{SF}$) and show how it relates to  advection-reaction PDEs. While this formulation remains computationally challenging, we'll show in Section \ref{sec:KSFD} how to approximate it in RKHS.

\subsection{Advection-Reaction with no Conservation of Mass}

\begin{definition}[Sobolev-Fisher Discrepancy] Let $p,q$ be two measures defined on $\pazocal{X}\subset \mathbb{R}^d$. For $\alpha>0$, the Sobolev-Fisher Discrepancy is defined as follows:
$$ \text{SF}(p,q)=\sup_{f} \left\{\mathbb{E}_{x\sim p} f(x) - \mathbb{E}_{x\sim q} f(x) : \quad \mathbb{E}_{x\sim q} \nor{\nabla_x f(x)}^2 + \alpha \mathbb{E}_{x\sim q} f^2(x) \leq 1 , \quad f |_{\partial \pazocal{X}}=0 \right\}$$.  
\label{def:SF}
 \end{definition}
 \vskip -0.2 in
Note that the objective of  $\text{SF}$ is an Integral Probability Metric (IPM) objective, and the function space imposes a constraint on the weighted Sobolev norm of the witness function $f$ on the support of the distribution $q$.
We refer to $q$ as the source distribution and $p$ as the target distribution.
The following theorem relates the solution of the Sobolev-Fisher Discrepancy to an advection-reaction PDE: 
\begin{theorem}[Sobolev-Fisher Critic as Solution of an Advection-Reaction PDE]
Let $u$ be the solution of the \emph{advection-reaction} PDE:
$$p(x)-q(x) = -div (q(x) \nabla_x u(x))+ \alpha u(x)q(x), \quad u|_{\partial \pazocal{X}}=0.$$
Then 
$\text{SF}^2(p,q)= \mathbb{E}_{x\sim q} \nor{\nabla_x u(x)}^2 + \alpha \mathbb{E}_{x\sim q} u^2(x)$,
with witness function
$f^*_{p,q}=u/\text{SF}(p,q)$.
\label{theo:AdvectionReaction}
\end{theorem}
From Theorem \ref{theo:AdvectionReaction} we see that the witness function of $\text{SF}^2$ solves an advection-reaction where the mass is transported from $q$  to $p$, via an advection term following the gradient flow of $\nabla_x u$, and a reaction term amounting to construction/destruction of mass that we also refer to as a birth-death process with a rate given by $u$.
Intuitively, if the witness function $u(x)>0$ we need to create mass, and destruct mass if $u(x)<0$. This is similar to the notion of particle birth and death defined in \cite{deathbirthJoan} and \cite{deathbirthLangevin}. 

In Proposition \ref{pro:UnconstForm} we give a convenient unconstrained  equivalent form for $\text{SF}^2$:
\begin{proposition}[Unconstrained Form of $\text{SF}^2$] $\text{SF}$ satisfies the expression:
$\text{SF}^2(p,q)= \sup_{u} L(u)$, with
$L(u)= 2( \mathbb{E}_{x\sim p} u(x) - \mathbb{E}_{x\sim q} u(x)) -\left( \mathbb{E}_{x\sim q} \nor{\nabla_x u(x)}^2 + \alpha \mathbb{E}_{x\sim q} u^2(x) \right).$
\label{pro:UnconstForm}
\end{proposition}

Theorem \ref{theo:KineticDeathBirth} gives a physical interpretation for $\text{SF}^2$ as finding the witness function $u$ that has minimum sum of  kinetic energy and rate of birth-death while transporting $q$ to $p$ via advection-reaction: 

\begin{theorem}[Kinetic Energy \& Birth-Death rates minimization]
Consider the following minimization:
$$P =\inf_{\substack{V: \pazocal{X}\to \mathbb{R}^d\\ r: \pazocal{X}\to\mathbb{R}}} \left\{\frac{1}{2}\left(\int_{\pazocal{X}} (\nor{V(x)}^2 +\alpha r^2(x))q(x)dx \right):
p(x) - q(x)= -div(q(x)V(x)) + \alpha r(x) q(x)\right\}$$
We then have that
$P=\frac{1}{2} \text{SF}^2(p,q)$, and moreover:
\begin{align*}
    \text{SF}^{ 2}(p,q)= \inf_{u} \int_{\pazocal{X}} \nor{\nabla_x u (x)}^2 q(x)dx + \alpha \int_{\pazocal{X}} u^2(x)q(x)dx, \\
    \text{subject to }  p(x) - q(x)= -div(q(x)\nabla_x u(x)) + \alpha u(x) q(x).
\end{align*}
\label{theo:KineticDeathBirth}
\end{theorem}

\vspace{-0.62cm}


\paragraph{Remarks.}
a) When $\alpha=0$ we obtain the Sobolev Discrepancy, or $\nor{p-q}_{\dot{H}^{-1}(q)}$, that linearizes the Wasserstein-$2$ distance. 
b) Note that this corresponds to a Beckman type of optimal transport \citep{peyre2017computational}, where we transport $q$ to $p$ ($q$ and $p$ do not have the same total mass) via an advection-reaction with mass not conserved. It is easy to see that
$\int_{\pazocal{X}} (p(x)-q(x))dx= \alpha \int_{\pazocal{X}} u(x)q(x) dx.$


\subsection{Advection-Reaction with Conservation of Mass}

Define the Sobolev-Fisher Discrepancy with conservation of mass:
$\overline{\text{SF}}^2(p,q)= \sup_{u} L(u),$ where
$L(u)= 2\left( \mathbb{E}_{x\sim p } u(x)-\mathbb{E}_{x\sim q} u(x)\right)- \left(\mathbb{E}_{x\sim q}\nor{\nabla_x u(x)}^2 + \alpha \left(\mathbb{E}_{x\sim q}\left(u(x) - \textcolor{blue}{\mathbb{E}_{x\sim q} (u(x))}\right)^2 \right)\right).$
The only difference between the previous expression and $\text{SF}^2$ in Proposition \ref{pro:UnconstForm} is that the variance of the witness function is kept under control, instead of the second order moment.
Defining $$\mathcal{E}(u)= \int_{\pazocal{X}}( \nor{\nabla_x u (x)}^2 +\alpha (u(x)- \mathbb{E}_{x\sim q}u(x) )^2 )q(x)dx$$
one can similarly show that $\overline{\text{SF}}$ has the primal representation:
$$\overline{\text{SF}}^2(p,q)= \inf_{u} \left\{\mathcal{E}(u):  p(x) - q(x)= -div(q(x)\nabla_x u(x)) + \alpha ( u(x)  - \mathbb{E}_{x\sim q} u(x)) q(x) \right\}.$$

\vspace{-0.2cm}

Hence, we see that $\overline{\text{SF}}$ is the minimum sum of kinetic energy and variance of birth-death rate for transporting $q$ to $p$ following an advection-reaction PDE with conserved total mass.
The conservation of mass comes from the fact that
$\chi(x)=-div(q(x)\nabla_x u(x)) + \alpha ( u(x) - \mathbb{E}_{x\sim q} u(x))q(x)$
satisfies: 
$$\int_{\pazocal{X}} (p(x)-q(x)) dx=\int_{\pazocal{X}}\chi(x)= 0 .$$

\section{Kernel Sobolev-Fisher Discrepancy}\label{sec:KSFD}
In this section we turn to the estimation of $\text{SF}$ discrepancy by restricting the witness function to a Reproducing Kernel Hilbert Space (RKHS), resulting in a closed-form solution.

\subsection{Estimation in Finite Dimensional RKHS}
Consider the finite dimensional RKHS, corresponding to an $m$ dimensional feature map $\Phi$:
 $$\mathcal{H}= \{f~ | f(x)=\scalT{w}{\Phi(x)} \text{ where } \Phi: \pazocal{X}\to \mathbb{R}^m , w\in \mathbb{R}^m\}.$$
Define the kernel mean embeddings $\mu(p)=\mathbb{E}_{x\sim p}\Phi(x), \mu(q)=\mathbb{E}_{x\sim q }\Phi(x)$, and $\delta_{p,q}=\mu(p)-\mu(q)$. Let  $C(q)= \mathbb{E}_{x\sim q}\Phi(x)\otimes \Phi(x)$  be the covariance matrix and $D(q)=\mathbb{E}_{x\sim q}J\Phi(x)^{\top}J\Phi(x)$ be the Gramian of the Jacobian, where $[J\Phi(x)]_{a,j}=\frac{ \partial \Phi_j(x)}{\partial x_a}$, $a=1\dots d, j=1\dots m$.

\begin{definition}[Regularized Kernel Sobolev-Fisher Discrepancy (KSFD)]
Let $u \in \mathcal{H}$, and let $\lambda >0$ and $\gamma \in \{0,1\}$, define:
$L_{\gamma,\lambda}(u)=2( \mathbb{E}_{x\sim p} u(x) - \mathbb{E}_{x\sim q} u(x)) -\left( \mathbb{E}_{x\sim q} [\nor{\nabla_x u(x)}^2 + \alpha(u(x)- \gamma \mathbb{E}_{q} u(x))^2 ]+\lambda \nor{u}^2_{\mathcal{H}}\right).$
The Regularized Kernel Sobolev-Fisher Discrepancy is defined as:
$$\text{SF}^2_{\mathcal{H},\gamma,\lambda}(p,q)=\sup_{u \in \mathcal{H}} L_{\gamma,\lambda}(u).$$
When $\gamma=0$ this corresponds to the unbalanced case, i.e.\ birth-death with no conservation of total mass, while for $\gamma=1$ we have birth-death with conservation of total mass.
\end{definition}

\begin{proposition}[Estimation in RKHS]
The Kernel Sobolev-Fisher Discrepancy is given by:
$\text{SF}^2_{\mathcal{H},\gamma,\lambda}(p,q)=\scalT{u^{\lambda,\gamma}_{p,q}}{\delta_{p,q}},$
where the critic 
$u^{\lambda,\gamma}_{p,q}= (D(q)+\alpha C_{\gamma}(q)+\lambda I_m)^{-1}\delta_{p,q},$
with $C_{\gamma}(q)= C(q)-\gamma \mu(q)\mu(q)^{\top}.$
Let $u^{\lambda,\gamma}_{p,q}(x)= \scalT{u^{\lambda,\gamma}_{p,q}}{\Phi(x)}$ and $\delta_{p,q}(x)= \scalT{\delta_{p,q}}{\Phi(x)}$, then:
$\nabla_{x}u^{\lambda,\gamma}_{p,q}(x)=(D(q)+\alpha C_{\gamma}(q)+\lambda I_m)^{-1}\nabla_{x}\delta_{p,q}(x).$
\label{pro:RKHScritic}
\end{proposition}

\paragraph{Remarks.}
a) For the unbalanced case $\gamma=0$, we refer to $\text{SF}^2_{\mathcal{H},0,\lambda}$ as $\text{SF}^2_{\mathcal{H},\lambda}$.
For the case of mass conservation $\gamma=1$, refer to  $\text{SF}^2_{\mathcal{H},1,\lambda}$ as $\overline{\text{SF}}^2_{\mathcal{H},\lambda}$. Note that
$C_{1}(q)=\bar{C}(q)= C(q) -\mu(q)\mu(q)^{\top}$.
b) A similar Kernelized discrepancy was introduced in \cite{Arbel:2018}, but not as an approximation of the Sobolev-Fisher discrepancy, nor in the context of unbalanced distributions and advection-reaction. c) For $\alpha=0$ we obtain the kernelized Sobolev Discrepancy KSD of \cite{SD}.

\subsection{ Kernel SF for Direct Measures}
Consider direct measures $p=\sum_{i=1}^N a_i \delta_{x_i}$ and $q=\sum_{j=1}^n b_j \delta_{y_j}$ (with no conservation of mass we can have $\sum_i a_i\neq \sum_j b_j\neq 1$ ).
An estimate of the Sobolev-Fisher critic is given by $\hat{u}^{\lambda,\gamma}_{p,q}= (\hat{D}(q)+\alpha \hat{C}_{\gamma}(q)+\lambda I_m)^{-1} (\hat{\mu}(p)-\hat{\mu}(q))$, where the empirical Kernel Mean Embeddings are $\hat{\mu}(p)=\sum_{i=1}^N a_i \Phi(x_i)$ and $\hat{\mu}(q)=\sum_{j=1}^n b_j \Phi(y_j)$. The empirical operator embeddings are given by $\hat{D}(q)= \sum_{j=1}^n b_j [J\Phi(y_j)]^{\top}J\Phi(y_j)$, and 
$\hat{C}_{\gamma}(q)= \sum_{j=1}^n b_j\Phi(y_j)\Phi(y_j)^{\top} -\gamma \hat{\mu}(q)\hat{\mu}(q)^{\top}.$

\section{Unbalanced Continuous Kernel Sobolev Descent}

Given the Kernel Sobolev-Fisher Discrepancy defined in the previous sections and its relation to advection-reaction transport, in this section we construct a Markov process that transports particles drawn from a source distribution to a target distribution. Note that we don't assume that the densities are normalized nor have same total mass.

%
%
%

\subsection{Constructing the  Continuous  Markov Process}
Given $\alpha,\lambda>0,\gamma \in\{0,1\}$ and  $n$ weighted particles  drawn from the source distribution :
$q^{n}_{0}=q=\sum_{i=1}^n b_{i}\delta_{y_i},$ i.e $X^0_i=y_i$ and $w^0_i=b_i$. Recall that the target distribution is given by $ p=\sum_{i=1}^N a_i \delta_{x_i}$.
We define the following Markov Process that we name Unbalanced Kernel Sobolev Descent:
\begin{align}
dX^i_{t}&= \nabla_{x}u^{\lambda,\gamma}_{p,q^{n}_{t}}(X^i_t)dt  \text{~~(advection step)}\nonumber\\
d w^i_{t}&=\alpha (u^{\lambda,\gamma}_{p,q^{n}_{t}}(X^{i}_t)- \gamma \mathbb{E}_{q^{(n)}_t} u^{\lambda,\gamma}_{p,q^{n}_t}(x)) w^{i}_t dt  \text{~~(reaction step)}\nonumber \\
 q^{n}_{t} &= \sum_{i=1}^n w^i_{t} \delta_{X^i_t},
 \label{eq:ContinuousUSD}
\end{align}
where $u^{\lambda,\gamma}_{p,q^{n}_{t}}$ is the critic of the Kernel Sobolev-Fisher discrepancy, whose expression and gradients are given in Proposition \ref{pro:RKHScritic}.
We see that USD consists of two steps: the advection step that updates the particles positions following the gradient flow of the Sobolev-Fisher critic, and a reaction step that updates the weights of the particles with a growth rate proportional to that critic.
This reaction step consists in mass construction or destruction, that depends on the confidence of the witness function.
This can be seen as birth-death process on the particles, where the survival $\log$ probability of a particle is proportional to the critic evaluation on this particle.



\subsection{Generator Expression and PDE in the limit of \texorpdfstring{$n\to \infty$}{Lg}}
Proposition \ref{pro:UnbalancedDescentEvolution} gives the evolution equation of a functional of the intermediate distributions $q^{n}_{t}$ produced in the descent, at the limit of infinite particles $n\to \infty$:  
\begin{proposition}  Let $\Psi: \pazocal{P}(\pazocal{X})\to \mathbb{R}$, be a functional on the probability space. Let  $q^{n}_{t}$ be the distribution produced by USD at time $t$. Let $q_{t}$ be its limit as $n\to \infty$, we have: 
$$ \partial_{t} \Psi [q_{t}]= (\pazocal{L}\Psi)[q_t],$$
where
$\pazocal{L} \Psi(q)= \int \scalT{ \nabla_x u^{\lambda,\gamma}_{p,q}(x) }{ \nabla_x D_{q} \Psi (x) } q(dx)+\alpha \int D_{q} \Psi(x)(u^{\lambda,\gamma}_{p,q}(x)-\gamma \mathbb{E}_{q} u^{\lambda,\gamma}_{p,q}) q(x) dx.$ 
Where the functional derivative $D_{\mu}$ is defined through first variation for a signed measure $\chi$ $(\int \chi(x) dx=0)$:
$$ \int D_{\mu} \Psi(x) \chi(x)dx = \lim_{\varepsilon \to 0} \frac{\Psi(\mu+\varepsilon \chi)- \Psi(\mu)}{\varepsilon}.$$
\label{pro:UnbalancedDescentEvolution}
\end{proposition}
\vskip -0.1in
In particular, the paths of USD in the limit of $n\to \infty$ satisfy the advection-reaction equation:
$$\partial_t q_{t}=-div(q_{t}\nabla_x u^{\lambda,\gamma}_{p,q_{t}})+ \alpha(u^{\lambda,\gamma}_{p,q_t}-\gamma \mathbb{E}_{q_t} u^{\lambda,\gamma}_{p,q}) q_t.$$

\subsection{Unbalanced Sobolev Descent decreases the \text{MMD}.}
The following Theorem shows that USD when the number of the particles goes to infinity decreases the MMD distance at each step, where: $\text{MMD}^2(p,q)= \nor{\mu(p)-\mu(q)}^2.$
\begin{theorem}[Unbalanced Sobolev Descent decreases the \text{MMD}]
\label{theo:mmd_decrease}
Consider the paths $q_{t}$ produced by USD. In the limit of particles $n\to \infty$ we have
\begin{align}
\frac{1}{2}\frac{d \text{MMD}^2(p,q_{t})}{dt}&=  - \left(\text{MMD}^2(p,q_{t})- \lambda \text{SF}^2_{\mathcal{H},\gamma,\lambda}(p,q_t) \right)\leq 0.
\label{eq:decrease}
\end{align}
\label{theo:mmdDecrease}
\end{theorem}
\vskip -0.2in
In particular, in the regularized case $\lambda>0$ with strict descent (i.e.\ $q_{t} \neq p$ implies $\text{MMD}^2(p,q_{t})-\lambda \text{SF}^2_{\mathcal{H},\gamma,\lambda}(p,q_t) > 0$), USD converges in the MMD sense: $\lim_{t\to \infty}\text{MMD}^2(p,q_{t})=0$.
Similarly to \cite{SD}, strict descent is ensured   if the kernel and the target distribution $p$ satisfy the  condition:
$\delta_{p,q} \notin \text{Null}(D(q)+\alpha C_{\gamma}(q)), \forall q \neq p.$

\textbf{USD Accelerates the Convergence.} We now prove a Lemma the can be used to show that Unbalanced Sobolev Descent has an acceleration advantage over Sobolev Descent \citep{SD}.

\begin{lemma}
\label{lem:accelerate}
In the regularized case $\lambda>0$ with $\alpha>0$, the Kernel Sobolev-Fisher Discrepancy $\text{SF}_{\mathcal{H},\gamma,\lambda}$ is strictly upper bounded by the Kernel Sobolev discrepancy $\mathcal{S}_{\mathcal{H},\lambda}$(i.e for $\alpha=0$) \citep{SD}: 
\begin{equation*}
    \text{SF}^2_{\mathcal{H},\gamma,\lambda}(p,q) < \mathcal{S}^2_{\mathcal{H},\lambda}(p,q).
\end{equation*}
\end{lemma}
From Lemma \ref{lem:accelerate} and  Eq. \eqref{eq:decrease}, we see that USD ($\alpha>0$), results in a larger decrease in MMD than SD \cite{SD} ($\alpha=0$), resulting in a steeper descent. Hence, USD advantages over SD are twofold: 1) it allows unbalanced transport, 2) it accelerates convergence for the balanced and unbalanced transport. 

\textbf{USD with Universal Infinite Dimensional Kernel.}
While we presented USD with a finite dimensional kernel for ease of presentation, we show in Appendix \ref{sec:infinite} that all our results hold for an infinite dimensional kernel. For a universal or a characteristic kernel, convergence in MMD implies convergence in distribution (see \cite[Theorem~12]{simongabriel2016kernel}). Hence, using a universal kernel, USD guarantees the weak convergence as $\text{MMD}(p,q_{t}) \to 0$.  




\subsection{Understanding the effect of the Reaction Step: Whitened Principal Transport Directions }
In \cite{SD} it was shown that the gradient of the Sobolev Discrepancy can be written as a linear combination of principal transport directions of the Gramian of derivatives $D(q)$. Here we show that unbalanced descent leads to a similar interpretation in a whitened feature space thanks to the $\ell_2$ regularizer.
Let $\tilde{\mathcal{H}}_{q}= \{f~ | f(x)=\scalT{v}{ \tilde{\Phi}_{q}(x)},  \tilde{\Phi}_{q}(x)= (C_{\gamma}(q)+ \frac{\lambda}{\alpha} I)^{-\frac{1}{2}}\Phi(x)\}$,
$\tilde{\delta}_{p,q}= (C_{\gamma}(q)+\frac{\lambda}{\alpha} I)^{-\frac{1}{2}} \delta_{p,q}$
$\tilde{D}(q)= (C_{\gamma}(q)+\frac{\lambda}{\alpha} I)^{-\frac{1}{2}}D(q) (C_{\gamma}(q)+\frac{\lambda}{\alpha} I)^{-\frac{1}{2}}$, and let
%
%
%
$v^{\lambda,\gamma}_{p,q}= (\tilde{D}(q)+\alpha I_m)^{-1}\tilde{\delta}_{p,q}.$
It is easy to see that the critic of the $\text{SF}$ can be written as: 
$u^{\lambda,\gamma}_{p,q}(x)= \scalT{u^{\lambda, \gamma}_{p,q}}{\Phi(x)}= \scalT{v^{\lambda,\gamma}_{p,q}}{\tilde{\Phi}_q(x)}$.
Note that $\tilde{\Phi}_q$ is a whitened feature map and $\tilde{D}(q)$ is the Gramian of its derivatives. Let $\tilde{d}_j,\lambda_j$ be the eigenvectors and eigenvalues of $\tilde{D}(q)$.
 We have:
$v^{\lambda,\gamma}_{p,q}= \sum_{j=1}^m \frac{1}{\lambda_j + \alpha} \tilde{d}_j \scalT{\tilde{d}_j}{\tilde{\delta}_{p,q}}$.
Hence, we write the gradient of the Sobolev-Fisher critic as
$\nabla_{x}u^{\lambda,\gamma}_{p,q}(x)= \sum_{j=1}^m \frac{1}{\lambda_j + \alpha }  \scalT{\tilde{d}_j}{\tilde{\delta}_{p,q}}  [J\tilde{\Phi}(x)] \tilde{d}_j  =   \sum_{j=1}^m \frac{1}{\lambda_j + \alpha }  \scalT{\tilde{d}_j}{\tilde{\delta}_{p,q}} \nabla_{x} \tilde{d}_j(x),$
where $\tilde{d}_{j}(x)=\scalT{\tilde{d}_j}{ \tilde{\Phi}_q(x)}$.
This says that the mass is transported along a weighted combination  of whitened principal  transport directions $\nabla_x \tilde{d}_j(x)$.
$\alpha$ introduces a damping of the transport as it acts as a spectral filter on the transport directions in the whitened space.

\section{Discrete time Unbalanced Kernel and Neural Sobolev Descent}

In order to get a practical algorithm in this Section we discretize the continuous USD given in Eq.\ \eqref{eq:ContinuousUSD}. We also give an implementation parameterizing the critic as a Neural Network.

\textbf{Discrete Time Kernel USD.}
Recall that the source distribution $q_{0}=q=\sum_{j=1}^n b_j\delta_{y_j}$, note $w^0_j=b_j$ and $x^0_j=y_j, j=1\dots n$. The target distribution $p=\sum_{j=1}^N a_j \delta_{x_j}$, and assume for simplicity $\sum_{j=1}^N a_j=1$. Let $\varepsilon>0$, for $\ell=1\dots L$, for $j=1\dots n$, we discretize the advection step:
$$x^{\ell}_j = x^{\ell-1}_j +\varepsilon \nabla_x u^{\lambda, \gamma}_{p,q_{\ell-1}}(x^{\ell-1}_j ). $$
Let $m_{\ell-1}= \sum_{j=1}^n w^{\ell-1}_ju^{\lambda,\gamma}_{p,q_{\ell-1}}(x^{\ell-1}_j).$
For $\tau >0$, similarly we discretize the reaction step as:
$$a^{\ell}_j =\log(w^{\ell-1}_j)  + \tau (u^{\lambda,\gamma}_{p,q_{\ell-1}}(x^{\ell-1}_j ) - \gamma m_{\ell-1} ). $$ 
If $\gamma=0$ (total mass not conserved) we define the reweighing  as follows: $w^{\ell}_j=\exp(a^{\ell}_j)$
and if $\gamma=1$ (mass conserved):
$w^{\ell}_j=\exp(a^{\ell}_j)/\sum_{i=1}^n \exp(a^{\ell}_i),$
and finally : $q^{\ell} = \sum_{j=1}^n w^{\ell}_j \delta_{x^{\ell}_j}$.

\textbf{Neural Unbalanced Sobolev Descent.} Motivated by the use of neural network critics in Sobolev Descent \citep{SD}, we propose a Neural variant of USD by parameterizing the critic of the Sobolev-Fisher Discrepancy as a Neural network $f_{\xi}$ trained via gradient descent with the Augmented Lagrangian Method (ALM) on the loss function of $\text{SF}$ given in Definition \ref{def:SF}.
The re-weighting is defined as in the kernel case above.
Neural USD with re-weighting is summarized in Algorithm \ref{alg:NSD} in Appendix \ref{app:alg}. 
Note that the re-weighting can also be implemented via a birth-death process as in \cite{deathbirthJoan}.
In this variant, particles are duplicated or killed with a probability driven by the growth rate given by the critic.
We give the details of the implementation as birth-death process in Algorithm \ref{alg:NSDDeathBirth} (Appendix \ref{app:alg}). 

\textbf{Computational and Sample Complexities.} The  computational complexity Neural USD is given by that of updating the witness function and particles by SGD  with backprop, i.e.\ $O(N (T+B))$, where $N$  is the mini-batch size, $T$ is the training time, $B$ is the gradient computation time for particles update. $T$ corresponds to a forward and a backward pass through the critic and its gradient. The sample complexity for estimating the Sobolev Fisher critic scales like $\nicefrac{1}{\sqrt{N}}$ similar to MMD \cite{MMD}.

\section{Relation to Previous Work}
Table \ref{table:summary} in Appendix \ref{app:Table} summarizes the main differences between Sobolev descent \citep{SD}, which only implements advection, and USD that also implements advection-reaction.
Our work is related to the conic particle descent that appeared in \cite{chizat2019sparse}  and \cite{deathbirthJoan}.
The main difference of our approach is that it is not based on the flow of a fixed functional, but we rather learn dynamically the flow that corresponds to the witness function of the Sobolev-Fisher discrepancy.
The accelerated Langevin Sampling of \cite{deathbirthLangevin} also uses similar principles in the transport of distributions via Langevin diffusion and a reaction term implemented as a birth-death process.
The main difference with our work is that in Langevin sampling the log likelihood of the target distribution is required explicitly, while in USD we only need access to samples from the target distribution.
USD relates to  unbalanced optimal transport \citep{chizat2015unbalanced, chizat2018interpolating, chizat2018scaling, sejourne2019sinkhorn} and  offers a computational flexibility when compared to Sinkhorn  approaches  \cite{chizat2018scaling, sejourne2019sinkhorn},  since it scales  linearly in the number of points while Sinkhorn is quadratic. Compared to WFR (Eq.\ \eqref{eq:BenamouWFR}), USD finds greedily the connecting path, while WFR solves an optimal planning problem. 


\section{Applications}
We experiment with USD on synthetic data, image coloring and prediction of developmental stages of scRNA-seq data.
In all our experiments we report the \text{MMD} distance with a gaussian kernel,  computed using  the random Fourier features (RF) approximation \cite{rahimi2007random} with $300$ RF and kernel bandwith equal to $\sqrt{d}$ (the input dimension). We consider the conservation of mass case, i.e.\ $\gamma=1$.    

\paragraph{Synthetic Examples.}

We test Neural USD descent (Algorithms \ref{alg:NSD} and \ref{alg:NSDDeathBirth}) on two synthetic examples.  In the first example (Figure \ref{fig:Gauss2MOG}), the source samples are drawn from a 2D standard Gaussian, while  target samples are drawn from a Mixture of Gaussians (MOG). Samples from this MOG have uniform weights.
In the second example (Figure \ref{fig:cat2heart}), source samples are drawn from a `cat'-shaped density whereas the target samples are drawn uniformly from a `heart'.
Samples from the targets have non-uniform weights following a horizontal gradient.  
In order to target such complex densities USD exploits advection and reaction by following the critic gradients and by creation and destruction of mass. 
We see in Figs \ref{fig:Gauss2MOG} and \ref{fig:cat2heart} a faster mixing of USD in both, implementation with weights \textcolor{blue}{\textbf{(w)}} and as birth-death \textcolor{blue}{\textbf{(bd)}} processes compared to the Sobolev descent algorithm of \cite{SD}. 
\begin{figure*}[ht!]
\begin{subfigure}{.65\textwidth}
  \centering
  \includegraphics[width=0.9\linewidth]{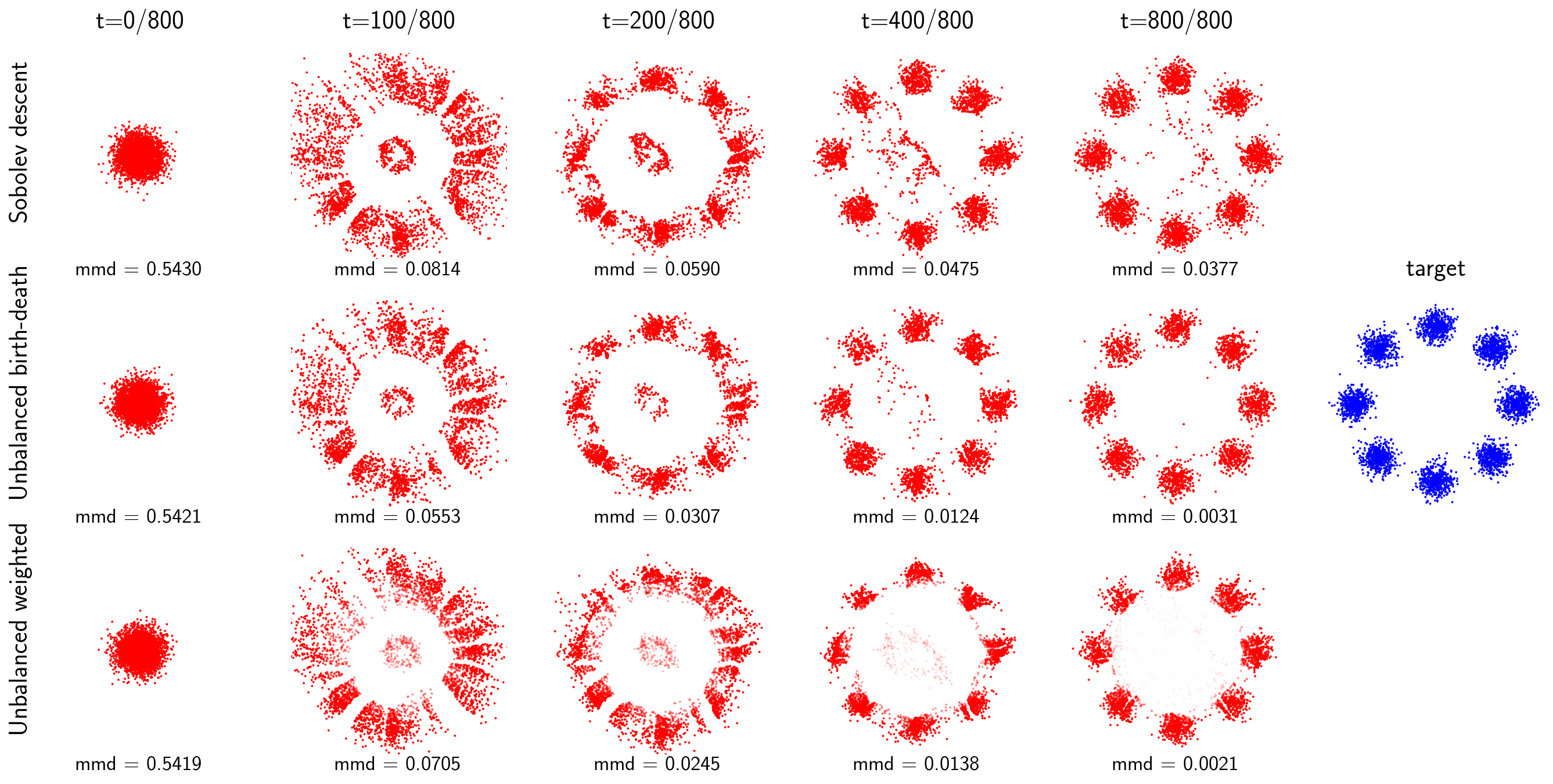}
  \caption{Neural USD paths in transporting a Gaussian to a MOG. We compare Sobolev descent (SD, \cite{SD}) to both USD implementations: with birth-death process (bd: Algorithm \ref{alg:NSDDeathBirth}) and weights (w: Algorithm \ref{alg:NSD}). USD outperforms SD in capturing the modes of the MOG.}
\end{subfigure}%
\hspace*{0.2in}
\begin{subfigure}{.3\textwidth}
\vspace{0.12in}
\centering
  \includegraphics[width=\linewidth]{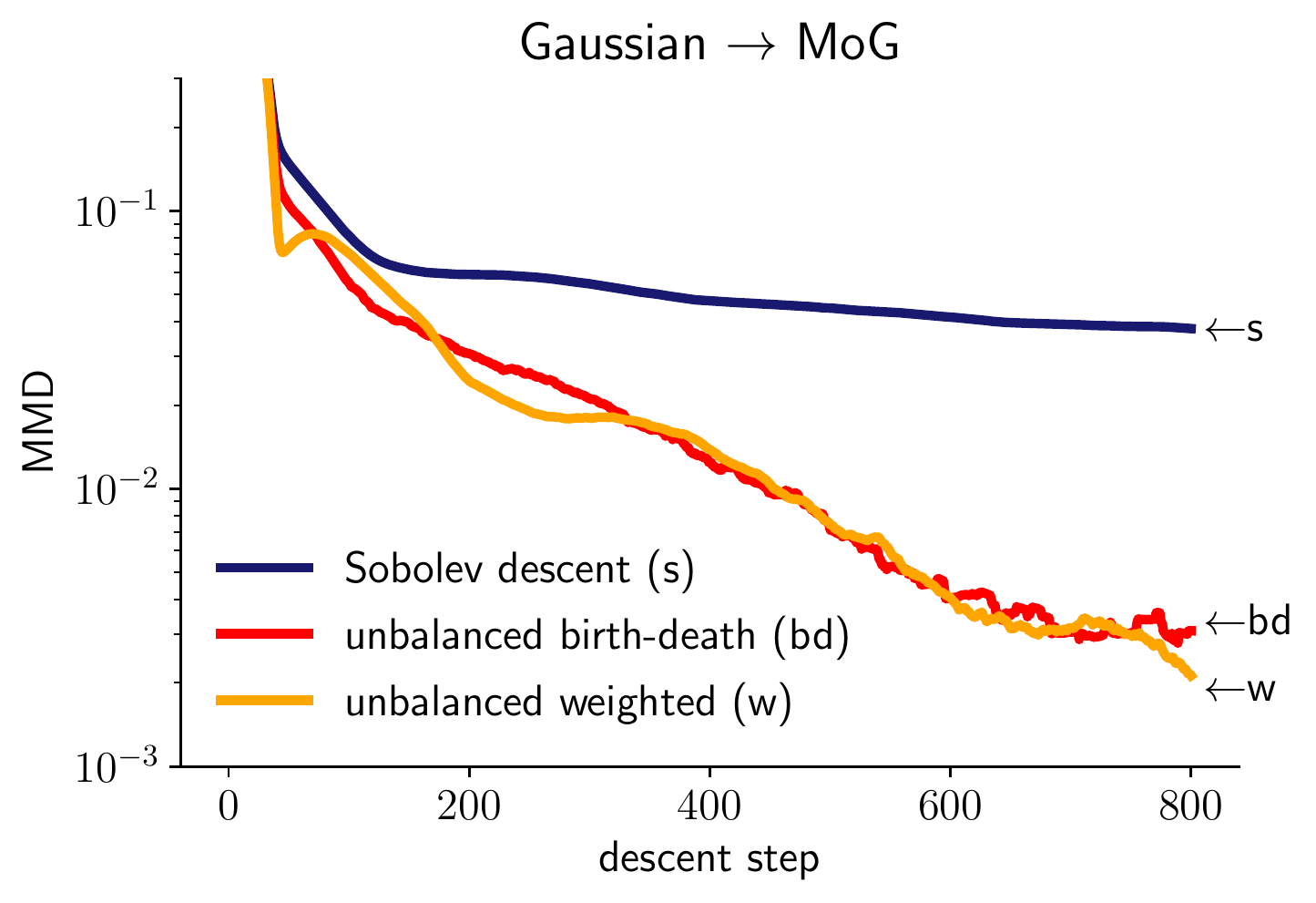}
  \caption{MMD as a function of step along the descent from a Gaussian to a MOG. Both USD implementations convergence faster to the target distribution, reaching lower MMD than Sobolev Descent that relies on advection only.}
\end{subfigure}
\caption{Neural USD transport of a Gaussian to a MOG (target distribution is uniformly weighted).}
\label{fig:Gauss2MOG}
\vskip -0.1in
\end{figure*}
\begin{figure*}[ht!]
\begin{subfigure}{.65\textwidth}
  \centering
  \includegraphics[width=0.9\linewidth]{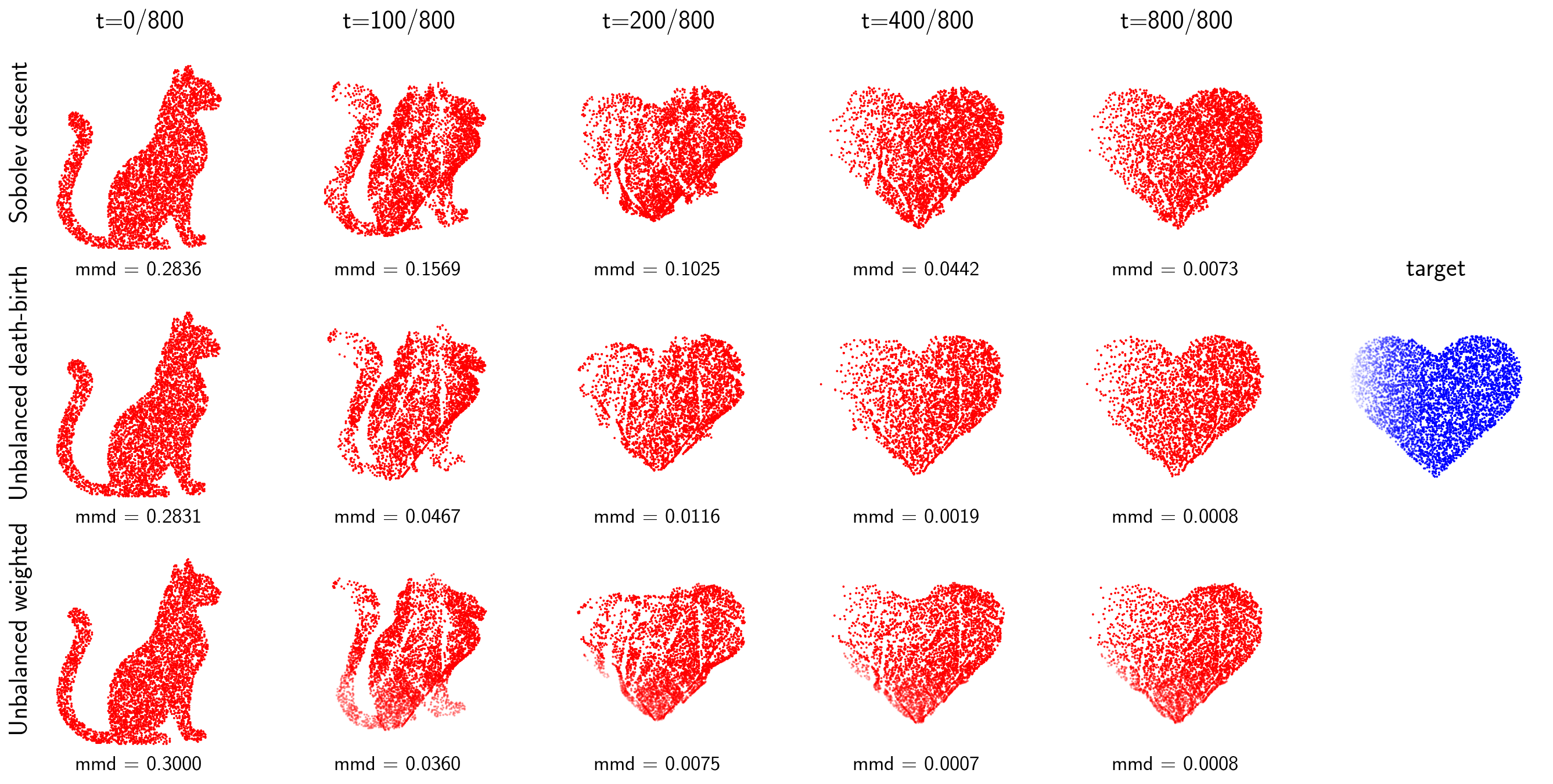}
  \caption{Neural USD transporting a `cat' distributed cloud to a `heart'. The main difference with the example above is that the points of the target distribution have non uniform weights describing a linear gradient as seen from the color code in the figure. Similarly to the MOG case, USD outperforms SD and better captures the non uniform density of the target.}
\end{subfigure}%
\hspace*{0.2in}
\begin{subfigure}{.3\textwidth}
\vspace{0.38in}
\centering
  \includegraphics[width=\linewidth]{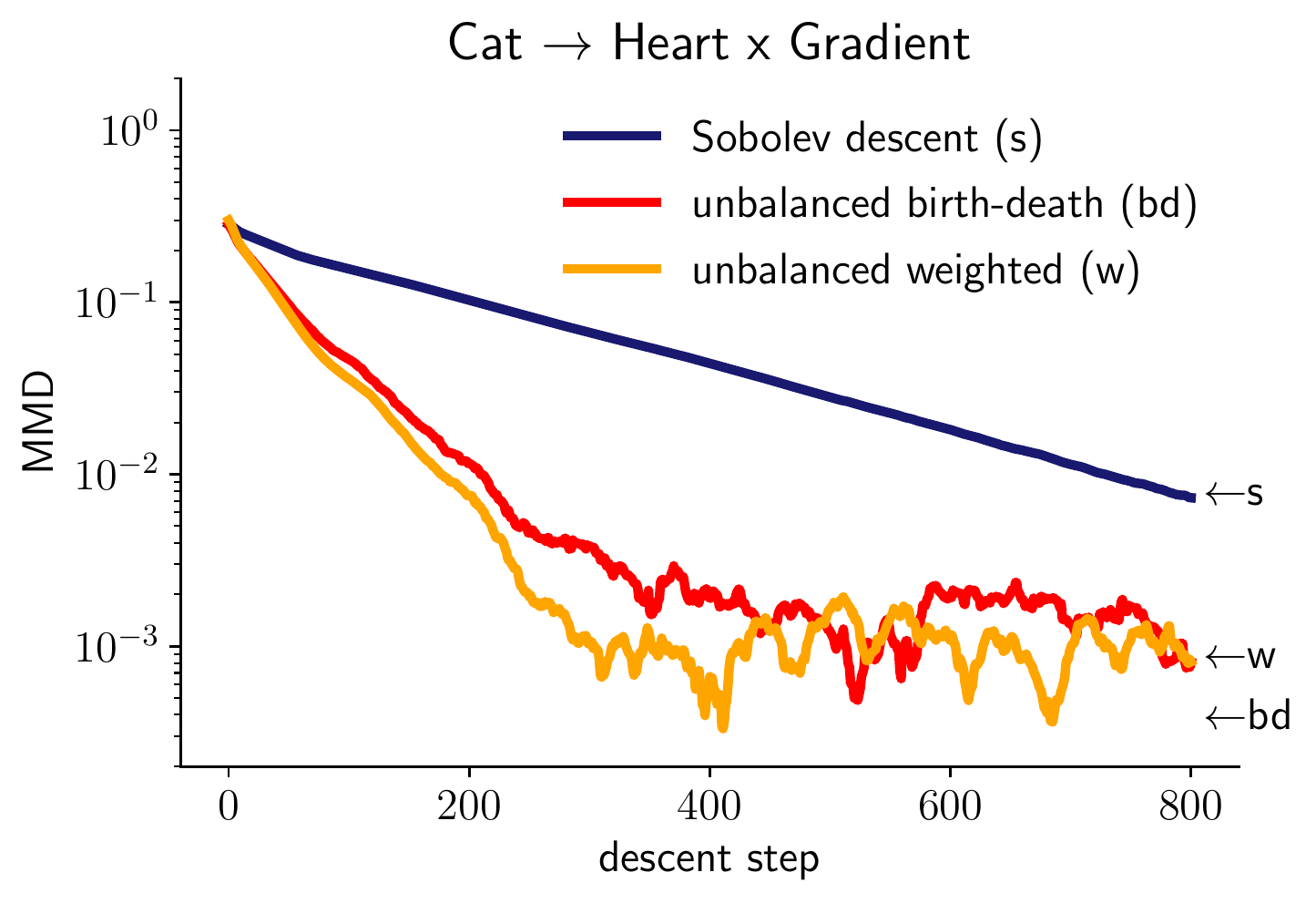}
  \caption{MMD as function of step along the descent from cat $\to$ heart $\times$ Grad. Similarly to the uniform target case USD accelerates the descent and outperforms SD.}
\end{subfigure}
\caption{Neural USD transport of a `cat' to a non-uniform `heart'. Samples from the target distribution have non-uniform weights given by $a_j$'s following a linearly decaying gradient.}
\label{fig:cat2heart}
\vspace{-0.4cm}
\end{figure*}

\paragraph{Image Color Transfer.}

We test Neural USD on the image color transfer task.
We choose target images that have sparse color distributions.
This is a good test for unbalanced transport since intuitively having birth and death of particles accelerates the transport convergence in this case.
We compare USD to standard optimal transport algorithms.
We follow the recipe of \cite{ferradans2013regularized} as implemented in the POT library \citep{flamary2017pot}, where images are subsampled for computational feasibility and then interpolated for out-of-sample points. We compare USD to Earth-Moving Distance (EMD), Sinkhorn \citep{cuturi2013sinkhorn} and Unbalanced Sinkhorn \citep{chizat2018scaling} baselines.
We see in Figure \ref{fig:imageColor}  that USD achieves smaller \text{MMD} to the target color distribution.  We give in Appendix \ref{app:color} in Fig \ref{fig:imageColorsequence} trajectories of the USD. 
\vspace{-0.2em}
\begin{figure*}[ht!]
\begin{subfigure}{\textwidth}
  \centering
  \includegraphics[width=0.9\linewidth]{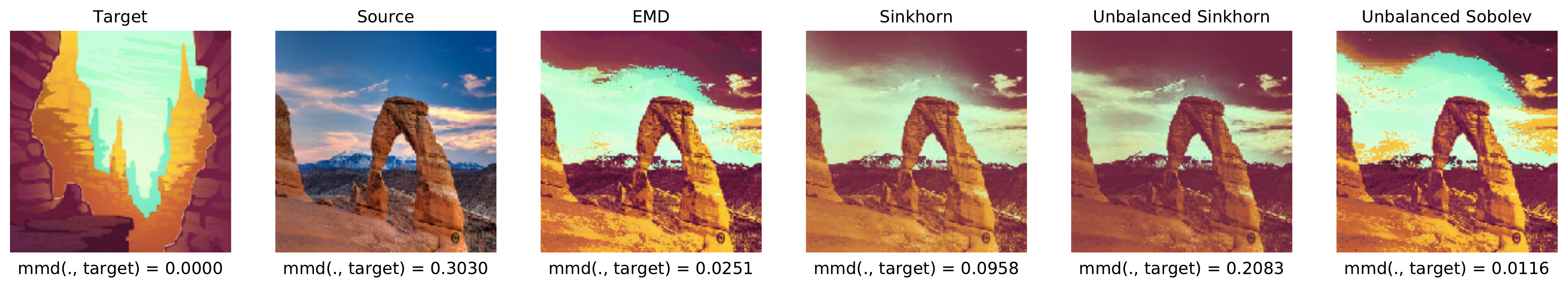}
\end{subfigure}%

\begin{subfigure}{\textwidth}
\centering
  \includegraphics[width=0.9\linewidth]{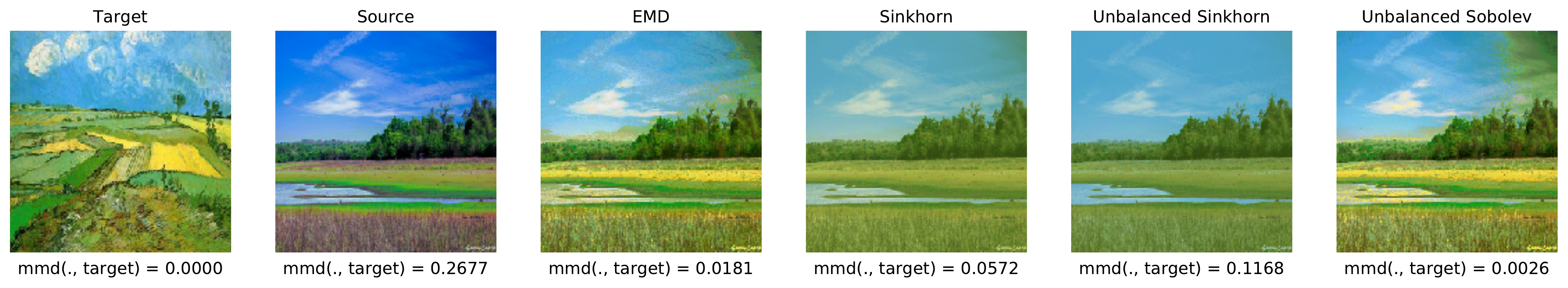}
\end{subfigure}
\caption{Color Transfer with USD using (bd) Algorithm \ref{alg:NSDDeathBirth}. Comparison to OT baselines (EMD, Sinkhorn and Unbalanced Sinkhorn). USD achieves lower MMD, and faithfully captures the sparse distribution of the target.}
\label{fig:imageColor}
\end{figure*}

\paragraph{Developmental Trajectories of Single Cells.}
\begin{figure}[ht!]
    \begin{minipage}[c]{0.35\textwidth}
        \includegraphics[width=1\linewidth]{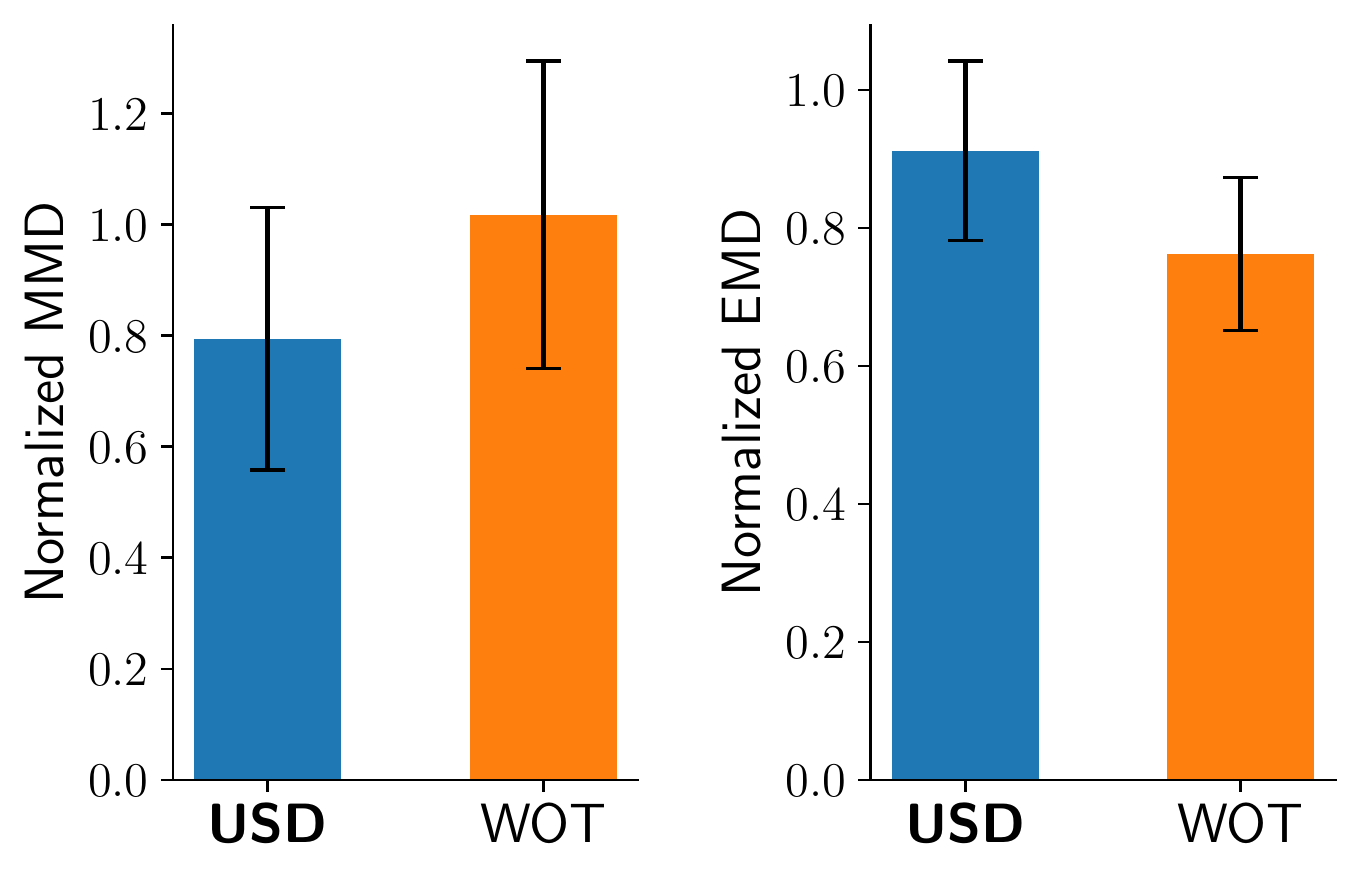}
    \end{minipage}\hfill
    \begin{minipage}[c]{0.60\textwidth}
        \caption{Mean and standard deviations plots of Normalized MMD and EMD for the intermediate stage prediction by USD and WOT (unbalanced OT) of \cite{schiebinger2019optimal} (means and standards deviation are computed over intervals). While USD outperforms WOT in MMD, the reverse holds in EMD. See text for an explanation.
        } \label{fig:wot}
    \end{minipage}
\vskip -0.30in
\end{figure}
When the goal is not only to transport particles but also to find intermediate points along trajectories, USD becomes particularly interesting.
This type of use case has recently received increased attention in developmental biology, thanks to single-cell RNA sequencing (scRNA-seq), a technique that records the expression profile of a whole population of cells at a given stage, but does so destructively.
In order to trace the development of cells in-between such destructive measurements, \cite{schiebinger2019optimal} proposed to use unbalanced optimal transport \citep{chizat2018scaling}.
Denoting those populations $q_{t_0}$ (source) and $q_{t_{1}}$ (target), then, in order to predict the population at an intermediate time $\frac{t_0+t_1}{2}$, \cite{schiebinger2019optimal} used a linear interpolation between matches between the source and target populations based on the coupling of unbalanced OT.
This type of interpolation is a form of McCann interpolate \cite{MCCANN1997}.
As an alternative, we propose to use the mid-point of the USD descent as an interpolate, i.e.\ the timestamp in the descent $t_{\nicefrac{1}{2}}$ such that $\text{MMD}(q_{t_{\nicefrac{1}{2}}},q_{t_0})= \text{MMD}(q_{t_{\nicefrac{1}{2}}},q_{t_1})$.
We test this procedure on the dataset released by \cite{schiebinger2019optimal}.
For all time  intervals $[t_0,t_1]$ in the dataset, we compute the intermediate stage $q_{t_{\nicefrac{1}{2}}}$.
We compare the quality of this interpolate with that obtained by the WOT algorithm of \cite{schiebinger2019optimal} in terms of MMD to the ground truth intermediate population $q^*_{t_{\nicefrac{1}{2}}}$, normalized by MMD between initial and final population, i.e.\ $\text{MMD}(q_{t_{\nicefrac{1}{2}}}, q^*_{t_{\nicefrac{1}{2}}}) /
\text{MMD}(q_{t_0},q_{t_1})$.
Fig.\ \ref{fig:wot} gives mean and standard deviation of the  normalized MMD between intermediate stages predicted by USD and the ground truth.
Note that mean and standard deviations are computed across $35$ time intervals, individual MMDs can be found in Figure \ref{fig:wotsupp} in Appendix \ref{app:plots}. 
From Figure \ref{fig:wot} we see that USD outperforms WOT in MMD, since USD is designed to decrease the MMD distance.
On the other hand, for fairness of the evaluation we also report Normalized EMD (Earth-Mover Distance, normalized similarly) for which WOT outperforms USD.
This is not surprising since WOT relies on unbalance OT, while USD instead provides guarantees in terms of MMD.

\vskip -0.1in
\section{Conclusion}
\vskip -0.15in

In this paper we introduced the KSFD discrepancy and showed how it relates to an advection-reaction transport.
Using the critic of KSFD, we introduced Unbalanced Sobolev Descent (USD) that consists in an advection step that moves particles and a reaction step that re-weights their mass.
The reaction step can be seen as birth-death process which, as we show theoretically, speeds up the descent compared to previous particle descent algorithms.
We showed that the MMD convergence of Kernel USD and presented two neural implementations of USD, using weight updates, and birth and death of particle, respectively.
We empirically demonstrated on synthetic examples and in image color transfer, that USD can be reliably used in transporting distributions, and indeed does so with accelerated convergence, supporting our theoretical analysis.
As a further demonstration of our algorithm, we showed that USD can be used to predict developmental trajectories of single cells based on their RNA expression profile.
This task is representative of a situation where distributions of different mass need to be compared and interpolated between, since the different scRNA-seq measurements are taken on cell populations of dissimilar size at different developmental stages. USD can naturally deal with this unbalanced setting. Finally we compared USD to unbalanced OT algorithms, showing its viability as a data-driven, more scalable dynamic transport method.


\section*{Broader Impact Statement}

Our work provides a practical particle descent algorithm that comes with a formal convergence proof and theoretically guaranteed acceleration over previous competing algorithms. Moreover, our algorithm can naturally handle situations where the objects of the descent are particles sampled from a source distribution descending towards a target distribution with different mass.

The type of applications that this enables range from theoretically principled modeling of biological growths processes (like tumor growth) and developmental processes (like the differentiation of cells in their gene expression space), to faster numerical simulation of advection-reaction systems.

Since our advance is mainly theoretical and algorithmic (besides the empirical demonstrations), its implications are necessarily tied to the utilization for which it is being deployed.
Beside the applications that we mentioned, particle descent algorithms like ours have been proposed as a paradigm to characterize and study the dynamics of Generative Adversarial Network (GANs) training. As such, they could indirectly contribute to the risks associated with the nefarious uses of GANs such as deepfakes. On the other hand, by providing a tools to possibly analyze and better understand GANs, our theoretical results might serve as the basis for mitigating their abuse.

\bibliography{refs,simplex}
\bibliographystyle{unsrt}



\appendix
\onecolumn
\begin{center}
\textbf{Supplementary Material: Unbalanced Sobolev Descent}
\end{center}

\subsection{Relation to Unbalanced Optimal Transport}
We now relate our definition of the Sobolev-Fisher discrepancy to the following norm.
For a signed measure $\chi$ define $\nor{\chi}_{\dot{H}^{-1,2}(\nu)}^2=$
\begin{align*} 
 \sup_{f,\int_{\pazocal{X}}(\nor{\nabla_x f(x) }^2+ \alpha f^2(x)) d\nu \leq 1 }\left|\int f d\chi\right| = \inf_{f, \chi(x)=-div(\nu(x)\nabla_x f(x))+\alpha f(x)\nu(x) } \int_{\pazocal{X}}(\nor{\nabla_x f }^2+ \alpha f^2) d\nu.
\end{align*}
It can be shown that $\text{SF}^2(p,q)= \nor{p-q}_{\dot{H}^{-1,2}(q)}^2$.

The dynamic formulation of the Wasserstein Fisher-Rao metric given in Equation \eqref{eq:BenamouWFR} can therefore be compactly written as: 
\begin{equation}
\text{WFR}^2(p,q)= \inf_{\nu_t}\int_{0}^1\nor{d\nu_{t}}_{\dot{H}^{-1,2}(\nu_{t})}^2.
\label{eq:WFRpaths}
\end{equation}
From this connection to WFR through $\nor{.}_{\dot{H}^{-1,2}(q)}^2$, we see the link of the Sobolev-Fisher discrepancy to unbalanced optimal transport, since it linearizes the WFR for small perturbations.

\section{Summary Table}
\label{app:Table}

\begin{table*}[ht!]
 	\resizebox{1\textwidth}{!}{	\begin{tabular}{|c|c|c|c|c|c|}
 			\hline
 			  & $\alpha$  & $\gamma$ &  \makecell{Markov Process\\ Particles $j=1\dots n$}& PDE (As $n\to \infty$)& \makecell{ Guarantee\\ $\frac{1}{2} \frac{d\text{MMD}^2(p,q_t)}{dt}=$
}  \\
			  \hline
			
			\hline
			
 			\makecell{\textbf{Sobolev Descent}\\ Flow of $\pazocal{S}_{\mathcal{H},\lambda}$\\Target: $p=\frac{1}{N}\sum_{i=1}^N \delta_{x_i}$\\ Source : $q=\frac{1}{n}\sum_{j=1}^n \delta_{y_j} $}	& $0$ & N/A& \makecell{~\\ $dX^j_{t}=\nabla_x u^{\lambda}_{p,q_{t}}(X^j_t)dt$\\ $q_{t}=\frac{1}{n}\sum_{j=1}^n \delta_{X^j_t} $\\ ~\\  Principal Transport Directions:\\  $ dX_{t}=\textcolor{blue}{\sum_{\ell=1}^m \frac{1}{\lambda_{\ell}+\lambda}\scalT{d_{\ell}}{\delta_{p,q_t}} \nabla_{x}d_{\ell}(x)}dt$\\ $(\lambda_j,d_j)=eig(D(q_{t}))$\\}  & \makecell{$\partial_{t}q_t=-div(q_t\nabla_x u^{\lambda}_{p,q_{t}})$\\~\\ Advection}  & \makecell{$-(\text{MMD}^2(p,q_t)-\lambda \pazocal{S}^2_{\mathcal{H},\lambda}(p,q_t))$ } \\
			
			\hline
 			\makecell{~\\ \textbf{Unbalanced Sobolev} \\ \textbf{Descent:} Flow of $\text{SF}_{\mathcal{H},\lambda}$ \\Target: $p=\sum_{i=1}^N a_i \delta_{x_i}$\\ Source : $q=\sum_{j=1}^n b_j\delta_{y_j} $ \\$(\sum_i a_i\neq \sum_j b_j)$\\~}	 & $\alpha>0$	& $\gamma=0$  &\makecell{ ~\\ $dX^j_{t}= \nabla_{x}u^{\lambda,\gamma}_{p,q_{t}}(X^j_t)dt$\\ $d w^j_{t}=\alpha (u^{\lambda,\gamma}_{p,q_{t}}(X^{j}_t)) w^{j}_t dt$\\  $q_{t} = \sum_{i=1}^n w^j_{t} \delta_{X^j_t}$\\ ~\\Whitened Principal Transport Directions : \\ $dX^j_{t}=\textcolor{blue} { \sum_{\ell=1}^m \frac{1}{\tilde{\lambda}_{\ell}+\alpha} \scalT{\tilde{d_{\ell}}}{\tilde{\delta}_{p,q}}\nabla_x \tilde{d}_{\ell}(X^j_t)} dt$ }  & \makecell{$\partial_{t}q_t=-div(q_t\nabla_x u^{\lambda,\gamma}_{p,q_{t}}) +\alpha u^{\lambda,\gamma}_{p,q_{t}}(x) q_{t}$\\~\\ Advection/Reaction\\ (Mass not conserved )}  & \makecell{$-(\text{MMD}^2(p,q_t)-\lambda \text{SF}^2_{\mathcal{H}}(p,q_t))$ }\\
			
			\hline
 			
			\makecell{\textbf{Balanced Sobolev}\\  \textbf{Descent:} Flow of $\overline{\text{SF}}_{\mathcal{H},\lambda}$ \\Target: $p=\sum_{i=1}^N a_i \delta_{x_i}$\\ Source : $q=\sum_{j=1}^n b_j\delta_{y_j} $ \\$(\sum_i a_i= \sum_j b_j)$\\~}	 & $\alpha>0$	& $\gamma=1$   &\makecell{~\\ $dX^j_{t}= \nabla_{x}u^{\lambda,\gamma}_{p,q_{t}}(X^j_t)dt$\\ $d w^j_{t}=\alpha (u^{\lambda,\gamma}_{p,q_{t}}(X^{j}_t) - \mathbb{E}_{q_{t}}u^{\lambda,\gamma}_{p,q_{t}}) w^{j}_t dt$\\  $q_{t} = \sum_{i=1}^n w^j_{t} \delta_{X^j_t}$ \\~\\ Whitened Principal Transport Directions : \\ $dX^j_{t}= \textcolor{blue} { \sum_{\ell=1}^m \frac{1}{\tilde{\lambda}_{\ell}+\alpha} \scalT{\tilde{d_{\ell}}}{\tilde{\delta}_{p,q}}\nabla_x \tilde{d}_{\ell}(X^j_t)} dt$ }    &\makecell{$\partial_{t}q_t=-div(q_t\nabla_x u^{\lambda,\gamma}_{p,q_{t}}) +\alpha( u^{\lambda,\gamma}_{p,q_{t}}(x) -\mathbb{E}_{q_t} u^{\lambda,\gamma}_{p,q_{t}}) q_{t}$\\~\\ Advection/Reaction\\ (Mass conserved )}  & \makecell{$-(\text{MMD}^2(p,q_t)-\lambda \overline{\text{SF}}^2_{\mathcal{H}}(p,q_t))$ } \\
			
			\hline
 		\end{tabular}}
 	 	\caption{Summary table comparing Unbalanced Sobolev Descent to Sobolev Descent.}
 	 	 	\label{table:summary}
 \end{table*}

\section{Algorithms} \label{app:alg}

\begin{algorithm}[ht!]
\caption{ \textcolor{blue}{\textbf{w}-}Neural Unbalanced Sobolev Descent (weighted version -- ALM Algorithm)}
\begin{algorithmic}
 \STATE {\bfseries Inputs:} $\varepsilon,\tau$ Learning rate particles, $n_c$ number of critics updates, $L$ number of iterations, $\gamma \in\{0,1\}$ \\
  $\{(a_i,x_i),i=1\dots N\}$, drawn from target distribution $\nu_p$\\
   $\{(b_j,y_j),j=1\dots n\}$ drawn from source distribution $\nu_q$\\
   Neural critic $f_{\xi}(x)=\scal{v}{\Phi_{\omega}(x)}$, $\xi=(v,\omega)$ parameters of the neural network\\
 \STATE {\bfseries Initialize} $x^0_j=y_j, w^0_j=b_j$ for $j=1\dots n$
 \FOR{ $\ell=1\dots L$}
\STATE \emph{\bfseries Critic Parameters Update}\\
\STATE (between particles updates, gradient descent on the critic is initialized from previous episodes)
\STATE $\xi \gets $ \textsc{Critic Update}($\xi$, target $\{x_i\}$, current source $\{(w^{\ell-1}_j,x^{\ell-1}_j)\}$,$\gamma$ )  (Given in Alg. \ref{alg:updatecritic} in Appendix \ref{app:alg}) \\
 \STATE \emph{\bfseries Particles and Weights Update} \\
 \FOR{$j=1$ {\bfseries to} $n$}
 \STATE $x^{\ell}_j = x^{\ell-1}_j +\varepsilon \nabla_x f_{\xi}(x^{\ell-1}_j ) $ (current $f_{\xi}$ is the critic between $q_{\ell-1}$ and $p$, advection step)
 \STATE $a^{\ell}_j =\log(w^{\ell-1}_j)  + \tau (f_{\xi}(x^{\ell-1}_j ) - \gamma m_{\xi}) $  (reaction step)
 \IF{$\gamma=1$ (mass conservation) }
 \STATE $w^{\ell} = \rm{Softmax}(a^{\ell}) \in \Delta_{n}$
 \ELSIF{$\gamma=0$ (mass not conserved)}
 \STATE $w^{\ell} = \exp(a^{\ell})$
 \ENDIF
 \ENDFOR
 \ENDFOR
 \STATE {\bfseries Output:} $\{(x^L_j, w^{L}_j), j=1\dots n\}$ 
 \end{algorithmic}
 \label{alg:NSD}
\end{algorithm}

\begin{algorithm}[ht!]
\caption{\textcolor{blue}{\textbf{bd}-}Neural Unbalanced Sobolev Descent (Birth-Death -- ALM Algorithm)}
\begin{algorithmic}
 \STATE {\bfseries Inputs:} Same inputs of Algorithm \ref{alg:NSD}
 \STATE {\bfseries Initialize} $x^0_j=y_j, w^0_j=\frac{1}{n}$ for $ j=1\dots n$
 \FOR{ $\ell=1\dots L$}
\STATE \emph{\bfseries Critic Parameters Update}\\
\STATE (between particles updates gradient descent on  the critic is initialized from previous episodes)
\STATE $\xi \gets $ \textsc{Critic Update}($\xi$, target $\{(a_i,x_i)\}$, current source $\{(\frac{1}{n},x^{\ell-1}_j)\}$,$\gamma$ ) (Given in Alg. \ref{alg:updatecritic} in App. \ref{app:alg})  \\
 \STATE \emph{\bfseries Particles and Weights Update (birth-death)}\\
 \FOR{$j=1$ {\bfseries to} $n$}
 \STATE $x^{\ell}_j = x^{\ell-1}_j +\varepsilon \nabla_x f_{\xi}(x^{\ell-1}_j ) $ (current $f_{\xi}$ is the critic between $q_{\ell-1}$ and $p$ )
  \STATE $m_{\xi} \gets \frac{1}{n}\sum_{i=1}^j f_{\xi}(x^{\ell}_i) + \frac{1}{n}\sum_{i=j+1}^{n} f_{\xi}(x^{\ell-1}_i)$
\IF{ $ \beta_j= f_{\xi}(x^{\ell}_j ) - \gamma m_{\xi} >0 $ }
\STATE Duplicate $x^{\ell}_j$ with probability $1-\exp(-\alpha \tau \beta_{j} )$
 \ELSIF{ $ \beta_j= f_{\xi}(x^{\ell}_j ) - \gamma m_{\xi}<0$}
 \STATE kill $x^{\ell}_j$ with probability $1-\exp(-\alpha \tau |\beta_{j}| )$
 \ENDIF
  \ENDFOR
  \COMMENT{Make population size $n$ again}
 \STATE  $n_{\ell}$ number of particles at the end of the loop
 \IF{$n_{\ell}>n$}
 \STATE Kill $n_{\ell}-n$ randomly selected particles
 \ELSIF{$n_{\ell}<n$}
 \STATE Duplicate $n-n_{\ell}$ randomly selected partciles 
\ENDIF
 \ENDFOR
 \STATE {\bfseries Output:} $\{(x^L_j), j=1\dots n\}$ 
 \end{algorithmic}
 \label{alg:NSDDeathBirth}
\end{algorithm}

\begin{algorithm}[H]
\caption{\textsc{Critic Update}($\xi$, target $\{(a_i,x_i)\}$, current source $\{(w^{\ell-1}_j,x^{\ell-1}_j)\}$,$\gamma$) }
\begin{algorithmic}
\FOR{$j=1$ {\bfseries to} $n_c$}
 \STATE $m_{\xi} \gets \sum_{j=1}^n w^{\ell-1}_j f_{\xi}(x^{\ell-1}_j)$
 \STATE  $\hat{\mathcal{E}}(\xi)\gets \sum_{i=1}^N a_i f_{\xi}(x_i) - m_{\xi}$
 \STATE $\hat{\Omega}(\xi)\gets \sum_j w^{\ell-1}_j \nor{\nabla_x f_{\xi}(x^{\ell-1}_j)}^2 + \alpha   \left(\sum_j w^{\ell-1}_j f^2_{\xi}(x^{\ell-1}_j) -\gamma m^2_{\xi}\right) $
 \STATE $\pazocal{L}_{S}(\xi,\lambda)= \hat{\mathcal{E}}(\xi)+ \lambda(1-\hat{\Omega}(\xi))-\frac{\rho}{2}(\hat{\Omega}(\xi)-1)^2$
  \STATE $(g_{\xi},g_{\lambda})\gets (\nabla_{\xi} {\pazocal{L}_{S}},\nabla_{\lambda}\pazocal{L}_{S})(\xi,\lambda) $
 \STATE $ \xi\gets \xi +\eta \text{ ADAM }(\xi,g_{\xi})$\\
 \STATE $\lambda \gets \lambda - \rho g_{\lambda}$ \COMMENT{SGD rule on $\lambda$ with learning rate $\rho$}
 \ENDFOR
\STATE {\bfseries Output:}  $\xi$
\end{algorithmic}
\label{alg:updatecritic}
\end{algorithm}

\section{Proofs}

\begin{proof} [Proof of Theorem \ref{theo:AdvectionReaction}]
Define the following dot product between $u, v $ in the the Sobolev Space: 
$$\scalT{u}{v}_{W^2_0}= \int_{\pazocal{X}} \scalT{\nabla_x u(x)}{\nabla_x v(x)} q(x) + \alpha \int_{\pazocal{X}} u(x)v(x) q(x) dx ,$$
and the norm :
$$\nor{u}^2_{W^2_0}=  \int_{\pazocal{X}} \nor{\nabla_x u(x)}^2 q(x) dx + \alpha \int_{\pazocal{X}} u^2(x) q(x) dx ,$$
Let $f$ be any function such that $f|_{\partial \pazocal{X}=0},$ and $\nor{f}_{W^2_0}\leq 1$:  
\begin{align*} 
\mathcal{E}(f)=&\int_{\pazocal{X}} f(x)(p(x)-q(x)) dx\\
&= - \int_{\pazocal{X}} f(x) div(q(x)\nabla_x u(x))dx + \alpha \int_{\pazocal{X}} u(x)f(x)q(x) \\
&= \int_{\pazocal{X}} \scalT{\nabla_x f(x)}{\nabla_x u(x)} q(x) + \alpha \int_{\pazocal{X}} u(x)f(x) q(x) dx\\
&= \scalT{u}{f}_{W^2_0} \text{ (By definition) }\\
&\leq \nor{u}_{W^2_0}\nor{f}_{W^2_0} \text{ (By Cauchy Schwarz) },\\
&\leq  \nor{u}_{W^2_0} \text{ ($f$ feasible, $\nor{f}_{W^2_0}\leq 1$) }
\end{align*}
Let $f^*_{p,q}= u/\nor{u}_{W^2_0}$, we have $\nor{f^*_{p,q}}_{W^2_0}=1$ and hence feasible, and it is easy to see that :
$$\mathcal{E}(f^*_{p,q})=  \nor{u}_{W^2_0},$$
and hence we have that for all $f$ feasible we have:
$$\mathcal{E}(f)\leq \mathcal{E}(f^*_{p,q}),$$
and hence $f^*_{p,q}$ achieves the $\sup$.
\end{proof}
\begin{proof}[Proof of Proposition \ref{pro:UnconstForm}] This can be easily proved using that $u^*$ solution of the PDE with source term is solution of that sup problem. $L(u^*)= \text{SF}^2(p,q)$ is clear from definition of $u^*$ we are left showing $L(u) \leq L(u*)$ for all $u$, this can be shown by proving that :
$$L(u)-L(u^*)= - \nor{u-u^*}^2_{W^2_0}\leq 0$$ and hence $L(u) \leq L(u^*)$ , hence $u^*$ achieves the sup.
\end{proof}

\begin{proof}[Proof of Theorem \ref{theo:KineticDeathBirth}]
Writing the  Lagrangian $u$ we have:

$$\inf_{V, r} \sup_{u}  \pazocal{L}(V,r,u) = \sup_{u} \inf_{V,r} \pazocal{L}(V,r,u),$$
where By convexity of the cost we exchange $\sup$ and  $\inf$ for
 $  \pazocal{L}(V,r,u) = \frac{1}{2} \int_{\pazocal{X}} \nor{V(x)}^2 q(x)dx + \alpha \frac{1}{2} \int_{\pazocal{X}} r^2(x)q(x)dx + \int_{\pazocal{X}} u(x) (p(x)-q(x)) - \int_{\pazocal{X}} \scalT{\nabla_x u(x)}{V(x)}q(x) -\alpha \int_{\pazocal{X}} r(x)u(x)q(x)$.\\
Note that $\inf_{V}   \int_{\pazocal{X}} \nor{V(x)}^2 q(x)dx  - \int_{\pazocal{X}} \scalT{\nabla_x u(x)}{V(x)}q(x) = -\sup_{V} \int_{\pazocal{X}} \scalT{\nabla_x u(x)}{V(x)}q(x)- \frac{1}{2}\int_{\pazocal{X}} \nor{V(x)}^2 q(x)dx= -\frac{1}{2}\int_{\pazocal{X}} \nor{\nabla_x u(x)}^2q(x)dx (\text{Fenchel Convex}) $. Similarly we have: $\inf_{r} \frac{1}{2}\int_{\pazocal{X}} r^2(x)q(x)dx - \int_{\pazocal{X}} r(x)u(x)q(x)= - \sup_{r}   \int_{\pazocal{X}} r(x)u(x)q(x) -\frac{1}{2}\int_{\pazocal{X}} r^2(x)q(x)dx = -\frac{1}{2}\int_{\pazocal{X}} u^2(x) q(x) dx .$
Hence the dual problem is : 
$$P= \sup_{u} \int_{\pazocal{X}} u(x)(p(x)-q(x)) dx - \frac{1}{2}\left( \int_{\pazocal{X}} \nor{\nabla_x u(x)}^2q(x)dx + \alpha \int_{\pazocal{X}} u^2(x) q(x) dx   \right)$$
 By Proposition 1 ,we have :
 $$P=\frac{1}{2} \text{SF}^2(p,q)$$
Hence $\text{SF}^2(p,q)$ has the equivalent form :
$$\text{SF}^2(p,q)= \inf_{V,r }  \int_{\pazocal{X}} \nor{V(x)}^2 q(x)dx + \alpha \int_{\pazocal{X}} r^2(x)q(x)dx $$
$$\text{ Subject to: }  p(x) - q(x)= -div(q(x)V(x)) + \alpha r(x) q(x)  $$
Since $V^*= \nabla_x u$ and $r^*=u$ we  have finally: 
$$\text{SF}^2(p,q)= \inf_{u} \int_{\pazocal{X}} \nor{\nabla_x u (x)}^2 q(x)dx + \alpha \int_{\pazocal{X}} u^2(x)q(x)dx $$
$$\text{ Subject to: }  p(x) - q(x)= -div(q(x)\nabla_x u(x)) + \alpha u(x) q(x).  $$

\end{proof}

\begin{proof}[Proof of Proposition \ref{pro:RKHScritic}] 
\begin{eqnarray*}
L_{\gamma,\lambda}(u)&=&2( \mathbb{E}_{x\sim p} u(x) - \mathbb{E}_{x\sim q} u(x)) -\left( \mathbb{E}_{x\sim q} [\nor{\nabla_x u(x)}^2 + \alpha(u(x)- \gamma \mathbb{E}_{q} u(x))^2 ]+\lambda \nor{u}^2_{\mathcal{H}}\right)\\
&=& 2\scalT{u}{\mu(p)-\mu(q)}_{\mathcal{H}} - \left( \scalT{u}{D(q) u}_{\mathcal{H}} + \alpha( \mathbb{E}_{q}u^2(x) - \gamma ( \mathbb{E}_{x\sim q} u(x))^2) +\lambda \nor{u}^2_{\mathcal{H}}  \right)\\
&=& 2\scalT{u}{\mu(p)-\mu(q)}_{\mathcal{H}} -   \left( \scalT{u}{D(q) u}_{\mathcal{H}} + \alpha ( \scalT{u}{C(q)u}_{\mathcal{H}}  - \gamma ( \scalT{u}{\mu(q)}_{\mathcal{H}})^2 +\lambda \nor{u}^2_{\mathcal{H}}  \right)\\
&=& 2\scalT{u}{\mu(p)-\mu(q)}_{\mathcal{H}}  - \scalT{u}{\left(D(q)+ \alpha (C(q)- \gamma \mu(q)\otimes \mu(q) )+\lambda I \right)u }_{\mathcal{H}}
\end{eqnarray*}
Setting first order optimality for the sup we obtain: 
$$\left(D(q)+ \alpha (C(q)- \gamma \mu(q)\otimes \mu(q) )+\lambda I \right)u^{\lambda, \gamma}_{p,q}= \mu(p)-\mu(q)=\delta_{p,q}.$$

\end{proof}

\begin{proof}[Proof of Proposition \ref{pro:UnbalancedDescentEvolution}] For simplicity we give here the proof for $\gamma =1$. $\gamma=0$ has a similar proof.
The proof follows ideas from \citep{deathbirthJoan}. 
Let $\Psi$ be a measure valued functional $\Psi: \mathcal{P}(\mathbb{R}^d)\to \mathbb{R}$. For a measure $\mu$, $\Psi(\mu)\in \mathbb{R}$.
The functional derivative $D_{\mu}$ is defined through first variation for a signed measure $\chi$ $(\int \chi(x) dx=0)$:
$$ \int D_{\mu} \Psi(x) \chi(x)dx = \lim_{\varepsilon \to 0} \frac{\Psi(\mu+\varepsilon \chi)- \Psi(\mu)}{\varepsilon}$$

A generator function is defined as follows for a measure valued markov process $\mu^{(n)}_t$ (defined with $n$ particles) is defined as follows:
$$  (\pazocal{L}_{n}\Psi)[\mu^{(n)}] = \lim_{s\to 0^+}\frac{\mathbb{E}_{\mu^{n}_0=\mu^{(n)}}(\Psi [\mu^{(n)}_s])- \Psi(\mu^{(n)})}{s} $$

where $$\mathbb{E}_{\mu^{n}_0=\mu^{(n)}}(\Psi [\mu^{(n)}_s]), $$
is the expectation of the functional $\Psi$ evaluated on the trajectory of the markov process $\mu^{(n)}_s$ taken on conditional on the initial step $\mu^{(n)}_0=\mu^{(n)}$.

\begin{enumerate}
\item Given our markov process i.e $\mu^{(n)}_{t}$ and $\mu^{(n)}_0$ we  find the expression of the generator $\pazocal{L}_{n}\Psi[\mu^{(n)}] $ (using pertrubation analysis )
\item Since the process is markovian letting $t\to 0$ and considering the generator it will give us the evolution also between $t$ and $t+dt$ of $\Psi[\mu^{(n)}_t]$:
$$\partial_{t} \Psi(\mu^{(n)}_t)=  (\pazocal{L}_{n}\Psi)[\mu^{(n)}_{t}], \Psi (\mu^{(n)}_{t})|_{t=0}=\Psi (\mu^{(n)}_{0}) $$
\item Consider $n\to \infty$, identify the PDE corresponding  to the generator
\end{enumerate}

As $s\to0$, and $\varepsilon \to 0$, we have:
$$E_{0}\Psi(q^{n}_s)- q^{n}= \underbrace{ E_{0}\Psi(q^{n}_s) -E_0 q^{n}_{s-\varepsilon}}_{\text{weights updates}}  + \underbrace{E_0 q^{n}_{s-\varepsilon}  - q^{n}}_{\text{advection}}$$

The advection part: 
\begin{align*}
A_n \Psi [q^{n}]&= \sum_{j=1}^n w_{j}\int \scalT{ \nabla_{x} u_{p,q^{(n)}}(X^j) \delta_{X_j}(dx)}{\nabla_x D_{q^{n}}\Psi (X^j)}\\
&= \int \scalT{\nabla_x u_{p,q^{n}}(x)}{\nabla_x D_{q^{n}} \Psi (x) }) q^n(dx)
\end{align*}
For the weight update part note that we have:
$$w^{j}_{s}= w^{j}_{s-\varepsilon}+\varepsilon \alpha  (u_{p,q^{n}_{s-\varepsilon}}(X^{j}_{s-\varepsilon})- \mathbb{E}_{q^{(n)}_{s-\varepsilon}} u_{p,q^{n}_{s-\varepsilon}}) w^{j}_{s-\varepsilon}  $$
$$q^{n}_{s}= \sum_{j=1}^N w^{j}_s \delta_{X^j_{s-\varepsilon}} $$
$$q^{n}_{s}= q^n_{s-\varepsilon}+ \varepsilon' \alpha \sum_{j=1}^n w^j_{s-\varepsilon} (u_{p,q^{n}_{s-\varepsilon}}(X^{j}_{s-\varepsilon})- \mathbb{E}_{q^{(n)}_{s-\varepsilon}} u_{p,q^{n}_{s-\varepsilon}}) \delta_{X^j_{s-\varepsilon}}  $$
Hence we have:
$$\frac{q^{n}_{s}(x)- q^{n}_{s-\varepsilon}(x)}{\varepsilon'}= \alpha (u_{p,q^{n}_{s-\varepsilon}}(x)- \mathbb{E}_{q^{(n)}_{s-\varepsilon}} u_{p,q^{n}_{s-\varepsilon}}) q^{n}_{s-\varepsilon}(x)= \chi$$
Hence the variation of $\Phi$:
$$\lim _{\varepsilon'\to 0} \frac{\Psi(q^{n}_{s})- \Psi(q^{n}_{s-\varepsilon})}{\varepsilon'}=\int D_{q^{n}_{s-\varepsilon}} \Psi(x) d\chi(x)=  \alpha \int D_{q^{n}_{s-\varepsilon}} \Psi(x)(u_{p,q^{n}_{s-\varepsilon}}(x)- \mathbb{E}_{q^{(n)}_{s-\varepsilon}} u_{p,q^{n}_{s-\varepsilon}}) q^{n}_{s-\varepsilon}(x) dx $$

As $s,\varepsilon \to 0 $ we obtain the effect of weights updates as follows:

$$W_n \Psi [q^{n}]=  \alpha \int D_{q^{n}} \Psi(x)(u_{p,q^{n}}(x)- \mathbb{E}_{q^{(n)}} u_{p,q^{n}}) q^{n}(x) dx$$

Hence the Generator has the following form:

$$ (\pazocal{L}_{n}\Psi)[q^{(n)}]=  \int \scalT{\nabla_x u_{p,q^{n}}(x)}{\nabla_x D_{q^{n}} \Psi (x) }) q^n(dx)+\alpha \int D_{q^{n}} \Psi(x)(u_{p,q^{n}}(x)- \mathbb{E}_{q^{(n)}} u_{p,q^{n}}) q^{n}(x) dx  $$

and we have: 
$$ \partial_{t} \Psi [q^n_{t}]= (\pazocal{L}_{n}\Psi)[q^{(n)}_t] , with q^{(n)}_0=q$$
As $n\to \infty$ we have the evolution of the PDE:
$$\partial_{t} q_{t}= -div(q(x)\nabla_{x}u_{p,q_{t}})+\alpha (u_{p,q_{t}}- \mathbb{E}_{q_{t}} u_{p,q_{t}})$$
and 
$$ \partial_{t} \Psi [q_{t}]= (\pazocal{L}\Psi)[q_t],$$
where
$\pazocal{L} \Psi(q)= \int \scalT{ \nabla_x u_{p,q}(x) }{ \nabla_x D_{q} \Psi (x) } q(dx)+\alpha \int D_{q} \Psi(x)(u_{p,q}(x)- \mathbb{E}_{q} u_{p,q}) q(x) dx. $

\end{proof}

\begin{proof}[Proof of Theorem \ref{theo:mmd_decrease} (Decrease of the \text{MMD} loss of the (Continous) Gradient Flow)]
For $u^{\gamma, \lambda}_{p,q_t}$ we omit the up-scripts $\gamma$ and $\lambda$ in the following.
Note that we have the following two expressions using the fact our functions are in the RKHS:
\begin{eqnarray*}
\int \scalT{ \nabla_x u_{p,q_{t}}(x) }{ \nabla_x \delta_{p,q_{t}} } q_{t}(dx)&=&\int \scalT{ u_{p,q_{t}} }{  (J\Phi(x))^{\top}J\Phi(x) \delta_{p,q_{t}} } q_{t}(dx)\\
&=&\scalT{u_{p,q_{t}}}{\mathbb{E}_{q_{t}}(J\Phi(x))^{\top} (J\Phi(x)) \delta_{p,q_{t}} }\\
&=&\scalT{u_{p,q_{t}}}{D(q_{t})\delta_{p,q_{t}} }.
\end{eqnarray*}
On the other hand:
\begin{eqnarray*}
 &&\int \delta_{p,q_{t}}(x)(u_{p,q_{t}}(x)- \gamma \mathbb{E}_{q_{t}} u_{p,q_{t}}) q_{t}(x) dx \\
  &=&\int \scalT{\delta_{p,q_{t}}}{\Phi(x)} \scalT{\Phi(x)- \gamma \mu(q_{t})}{u_{p,q_{t}}} q_{t}(x)dx\\
 &=& \int \scalT{\delta_{p,q_{t}}}{\Phi(x) -\gamma \mu(q_{t})} \scalT{\Phi(x)- \gamma \mu(q_{t})}{u_{p,q_{t}}} q_{t}(x)dx \\
&+&  \gamma \int \scalT{\delta_{p,q_{t}}}{ \mu(q_{t})} \scalT{\Phi(x)- \gamma \mu(q_{t})}{u_{p,q_{t}}} q_{t}(x)dx\\
&=&\scalT{\delta_{p,q_{t}}}{(\int (\Phi(x)-\gamma \mu(q_{t}))\otimes (\Phi(x)-\gamma \mu(q_{t})) q_{t}(dx)) u_{p,q_{t}} } \\
&+& \gamma \scalT{\delta_{p,q_{t}}}{ \mu(q_{t})} \int\scalT{\Phi(x)-\gamma \mu(q_{t})}{u_{p,q_{t}}} q_{t}(x)dx\\
&=& \scalT{\delta_{p,q_t }}{C_{\gamma}(q_{t}) u_{p,q_t }}  + \underbrace{\gamma \scalT{\delta_{p,q}}{ \mu(q_{t})} \scalT{\mu(q_{t})- \gamma \mu(q_{t})}{u_{p,q_{t}}}}_{=0 , \text{ for } \gamma \in \{0,1\}}\\
&=& \scalT{\delta_{p,q_t}}{C_{\gamma}(q_t) u_{p,q_{t}}} +0.
\end{eqnarray*}
Consider $\Psi(q)= \frac{1}{2}\text{MMD}^2(p,q)= \frac{1}{2}\nor{\mu(p)-\mu(q)}^2$, it is easy to see that the functional derivative wrt to $q$ is $D_{q}\Psi(q)(x)= -\delta_{p,q}$.
Hence we have:
\begin{align*}
\frac{1}{2}\frac{d \text{MMD}^2(p,q_{t})}{dt}&= - \int \scalT{ \nabla_x u_{p,q_{t}}(x) }{ \nabla_x \delta_{p,q_{t}} } q_{t}(x)dx - \alpha \int \delta_{p,q_{t}}(x)(u_{p,q_{t}}(x)- \gamma \mathbb{E}_{q_{t}} u_{p,q_{t}}) q_{t}(x) dx \\
& = -\scalT{\delta_{p,q_{t}}}{D(q_{t})u_{p,q_{t}}} - \alpha \scalT{\delta_{p,q_{t}}}{ C_{\gamma}(q_{t}) u_{p,q_{t}}}\\
& = -\scalT{\delta_{p,q_{t}}}{(D(q_{t})+ \alpha {C}_{\gamma}(q_{t}) + \lambda I - \lambda I ) u_{p,q_{t}}}\\
& = - (\scalT{\delta_{p,q_{t}}}{(D(q_{t})+ \alpha {C}_{\gamma}(q_{t}) + \lambda I)u_{p,q_{t}}  } - \lambda \scalT{\delta_{p,q_{t}}}{u_{p,q_{t}}})\\
&= -\left(\scalT{\delta_{p,q_{t}}}{\delta_{p,q_{t}}} - \lambda \scalT{\delta_{p,q_{t}}}{u_{p,q_{t}}}\right) \text{where we used that  }  (D(q_{t})+ \alpha {C}_{\gamma}(q_{t}) + \lambda I)u_{p,q_{t}}=\delta_{p,q_{t}}\\
& = - \left(\text{MMD}^2(p,q_{t})- \lambda \text{SF}^2_{\mathcal{H},\gamma,\lambda}(p,q_t)\right) \text{by Definition of Sobolev-Fisher Distance }\\
&\leq 0
\end{align*}
since $$\text{MMD}^2(p,q_{t})\geq \lambda \text{SF}^2_{\mathcal{H},\gamma,\lambda}(p,q_t)$$


\end{proof}

We now prove a Lemma the can be used to show that Unbalanced Sobolev descent has an acceleration advantage over Sobolev descent \citep{SD}.

\begin{lemma}
\label{lem:acceleration}
In the regularized case $\lambda>0$ with $\alpha>0$, the Kernel Sobolev-Fisher Discrepancy $\text{SF}_{\mathcal{H},\gamma,\lambda}$ is strictly upper bounded by the Kernel Sobolev discrepancy $\mathcal{S}_{\mathcal{H},\lambda}$ \citep{SD}: 
\begin{equation*}
    \text{SF}^2_{\mathcal{H},\gamma,\lambda}(p,q) < \mathcal{S}^2_{\mathcal{H},\lambda}(p,q).
\end{equation*}
\end{lemma}

\begin{proof}
Recall that (see Proposition \ref{pro:RKHScritic}):

$$\text{SF}^2_{\mathcal{H},\gamma,\lambda}(p,q)=\scalT{(D+\alpha C_{\gamma}+ \lambda I_m)^{-1}\delta_{p,q}}{\delta_{p,q}},$$

and that (see \cite{SD}):
$$\mathcal{S}^2_{\mathcal{H},\lambda}(p,q)=\scalT{(D+ \lambda I_m)^{-1}\delta_{p,q}}{\delta_{p,q}}.$$

We now make use of the \emph{Woodbury identity} $(A+B)^{-1} = A^{-1} - (A + A B^{-1} A)^{-1}$ with $A=D+\lambda I_m$ and $B=\alpha C_{\gamma}$,
which allows us to write:

\begin{equation}
    (D+\alpha C_{\gamma}+ \lambda I_m)^{-1} = (D+\lambda I_m)^{-1} - E,
\label{eq:wood}
\end{equation}
where $E=(A + A B^{-1} A)^{-1}$.

Notice that $A=D+\lambda I_m$ and $B=\alpha C_{\gamma}$ are both symmetric positive definite (SPD).
Because the inverse of a SPD matrix is itself a SPD matrix, $B^{-1}$ is SPD.
Because the product of SPD matrices is itself SPD, $A B^{-1} A$ is SPD.
Because the inverse of the sum of SPD matrices is itself SPD, $E$ is SPD.

Equation \eqref{eq:wood} then implies:
$$(D+\alpha C_{\gamma}+ \lambda I_m)^{-1} \prec (D+\lambda I_m)^{-1},$$
which, together with the definitions of $\text{SF}^2_{\mathcal{H},\gamma,\lambda}$ and $\mathcal{S}^2_{\mathcal{H},\lambda}$, concludes the proof.
\end{proof}
\section{Unbalanced Sobolev Descent With a Universal Kernel }\label{sec:infinite}
While we presented the paper in a finite dimensional RKHS, to ease the presentation. We show in this Section, that  our theory is general and apply to the infinite dimensional case. Of interest to us, is the case of a universal kernel. The convergence in $\mathrm{MMD}$  for a universal kernel implies the weak convergence in the distributional sense.

\subsection{ Kernel Mean Embeddings, Covariance and Grammian of Derivatives Operators  }
Let $\mathcal{H}$ be a Reproducing Kernel Hilbert Space with an associated kernel $k:\pazocal{X}\times \pazocal{X}\to \mathbb{R}^{+}$. We make the following assumptions on $\mathcal{H}$ as in \cite{SD}:
\begin{enumerate}
\item[A1] There exists $\kappa_1<\infty$ such that $\sup_{x\in \pazocal{X}} \nor{k_x}_{\mathcal{H}}<\kappa_1$.
\item [A2] The kernel is $C^2(\pazocal{X}\times \pazocal{X})$ and there exists $\kappa_2<\infty$ such that for all $a=1\dots d$:\\
$\sup_{x\in \pazocal{X}} Tr((\partial_a k )_x\otimes (\partial_a k )_x )<\kappa_2$.
\item [A3] $\mathcal{H}$ vanishes on the boundary (assuming $\pazocal{X}=\mathbb{R}^d$ it is enough to  have for $f$ in $\mathcal{H}$  $\lim_{\nor{x}\to \infty} f(x)=0$).
\end{enumerate}  
 The reproducing property give us that $f(x)=\scalT{f}{k_x}_{\mathcal{H}}$ moreover $(D_a f)(x)=\frac{\partial}{\partial x_a}f(x)=\scalT{f}{(\partial_ak)_x}_{\mathcal{H}} $, where $(\partial_a k)_x(t)=\scalT{\frac{\partial k(s,.)}{\partial s_a}\big|_{s=x}}{k_t}$. Note that those two quantities  ($f(x)$ and $(D_a f)(x)$) are well defined and bounded thanks to assumptions A1 and A2  \cite{derivativesRKHS}. \\
Similar to finite dimensional case we define the  Gramian  of derivatives operator of a distribution $q$ :
\begin{equation}
D(q)=  \mathbb{E}_{x\sim \nu_q} \sum_{a=1}^d (\partial_a k)_x\otimes (\partial_a k)_x ~ D(\nu_q)\in \mathcal{H}\otimes \mathcal{H}
\label{eq:Dop}
\end{equation}
The Kernel mean embedding is defined as follows:
\begin{equation}
\mu(p)=\mathbb{E}_{x\sim \nu_p}k_{x} \in \mathcal{H}.
\label{eq:kme}
\end{equation}
The covariance operator is defined as follows for $\gamma \in \{0,1\}$:
\begin{equation}
C_{\gamma}(q)=\mathbb{E}_{x\sim q} k_{x}\otimes k_{x} -\gamma\mu(q)\otimes \mu(q) 
\label{eq:Cop}
\end{equation}
\subsection{Regularized Kernel Sobolev Fisher Discrepancy }
Let $\lambda>0,\alpha \geq 0$, similarly the Kernel Sobolev Fisher Discrepancy has the following form:
$$ \text{SF}^2_{\mathcal{H},\gamma,\lambda}(p,q)=\nor{(D(q)+\alpha C_{\gamma}(q)+\lambda I)^{-\frac{1}{2}} (\mu(\nu_p)-\mu(\nu_q))}^2_{\mathcal{H}},$$
where $D(q),\mu(q), C_{\gamma}(q)$ are defined in Equations \eqref{eq:Dop},\eqref{eq:kme} and \eqref{eq:Cop} respectively.  
The Sobolev Fisher witness function  is defined as follows:
$$u^{\lambda,\gamma}_{p,q}=(D(q)+\alpha C_{\gamma}(q)+\lambda I)^{-1}(\mu(\nu_p)-\mu(\nu_q)) \in \mathcal{H}$$
its evaluation function is 
$$u^{\lambda,\gamma}_{p,q}(x)=\scalT{(D(\nu_q)+\lambda I)^{-1}(\mu(\nu_p)-\mu(\nu_q))}{k_x}_{\mathcal{H}}$$ and its derivatives for $a=1\dots d$:
$$\partial_{a} u^{\lambda,\gamma}_{p,q}(x)=\scalT{(D(\nu_q)+\lambda I)^{-1}(\mu(\nu_p)-\mu(\nu_q))}{\partial_{a}k_x}_{\mathcal{H}}.$$
\subsection{USD with Infinite dimensional Kernel decreases the MMD 
distance}
Theorem \ref{theo:mmd_decrease}  holds for the infinite dimensional case. To see that it is enough to replace in the proof of Theroem \ref{theo:mmd_decrease} finite dimensional operators and embeddings $D(q),C_{\gamma}(q),\mu(q)$ with their infinite dimensional counterparts given in Equation in Equations \eqref{eq:Dop},\eqref{eq:kme} and \eqref{eq:Cop}. All norms and dot products in $\mathbb{R}^m$, are also to be replaced with $\nor{.}_{\mathcal{H}}$ and $\scalT{.}{.}_{\mathcal{H}}$. 


\section{Code and Hyper-parameters}
\begin{lstlisting}[language=Python, caption={Pytorch code for computing cost function $\pazocal{L}_{S}(\xi,\lambda)$ in Algorithm \ref{alg:updatecritic}}]
import torch
from torch.autograd import grad

def descent_cost(f, x_p, w_p, x_q, w_q, lambda_aug, alpha, rho, gamma=1):
    """Computes the objective of Unbalance Sobolev Descent and returns the loss = -obj
    """
    x_q.requires_grad_(True)

    f_p, f_q = f(x_p), f(x_q)
    Ep_f = (w_p * f_p).mean()
    Eq_f = (w_q * f_q).mean()

    # FISHER
    constraint_F = (w_q * f_q**2).mean() - gamma * Eq_f**2

    # SOBOLEV
    grad_f_q = grad(outputs=Eq_f, inputs=x_q, create_graph=True)[0]
    normgrad_f2_q = (grad_f_q**2).sum(dim=1, keepdim=True)
    constraint_S = (w_q * normgrad_f2_q).mean()

    # Combining FISHER and SOBOLEV constraints
    constraint_tot = (constraint_S + alpha * constraint_F - 1.0)

    obj_f = Ep_f - Eq_f \
            - lambda_aug * constraint_tot - rho/2  * constraint_tot**2

    return -obj_f, Ep_f, Eq_f, normgrad_f2_q
\end{lstlisting}

\section{Architecture of Neural Network discriminator}

\begin{lstlisting}[language=Python, caption={}]
D_mlp = Sequential(
    (L0): Linear(in_features=n_inputs, out_features=n_layers[0], bias=True)
    (N0): ReLU(inplace=True)
    (L1): Linear(in_features=n_layers[0], out_features=n_layers[1], bias=True)
    (N1): ReLU(inplace=True)
    (D1): Dropout(p=0.2, inplace=False)
    (L2): Linear(in_features=n_layers[1], out_features=n_layers[2], bias=True)
    (N2): ReLU(inplace=True)
    (V): Linear(in_features=n_layers[2], out_features=1, bias=False)
)
\end{lstlisting}

\section{Hyperparameters for experiments}

\begin{lstlisting}[language=Python, caption={Hyperparameters for synthetic experiments (Figs.\ \ref{fig:Gauss2MOG}, \ref{fig:cat2heart}, \ref{fig:Gauss2circles}, \ref{fig:disk2heart})}]
{
    "n_layers": [64, 1024, 64],     # Number of neurons in hidden layers of discriminator
    "n_points_src": 4000,           # Number of points sampled from source distribution
    "n_points_target": 4000,        # Number of points sampled from target distribution
    "T": 800,                       # Number descent steps
    "optimizer": Adam(amsgrad=True)     # Optimizer for discriminator (reset at every update of distribution q)
    "batchSize": 512,               # Batch size for discriminator updates
    "n_c_startup": 200,             # Number of steps for discriminator updates at startup
    "n_c": 20,                      # Number of steps for discriminator updates in-between updates of distribution q
    "wdecay": 1e-5,                 # Weight decay factor
    "lrD": 1e-4,                    # Learning rate for discriminator updates
    "lrQ": 1e-4,                    # Learning rate for updates of distribution q
    "tau": 1e-3,                    # Birth-death rate
    "alpha": 0.6,                   # Damping factor ($\alpha$ in Algorithm 3)
    "lambda_aug_init": 1e-5,        # Initialization of augmented Lagrange multiplier (in Algorithm 3)
    "rho": 1e-6                     # Learning rate of augmented Lagrange multiplier
}
\end{lstlisting}

\begin{lstlisting}[language=Python, caption={Hyperparameters for color transfer experiments (Figs.\ \ref{fig:imageColor}, \ref{fig:imageColorsequence})}]
{
    "n_layers": [128, 2048, 128],   # Number of neurons in hidden layers of discriminator
    "n_points_src": 65536,           # Number of points sampled from source distribution
    "n_points_target": 65536,        # Number of points sampled from target distribution
    "T": 800,                       # Number descent steps
    "optimizer": Adam(amsgrad=True)     # Optimizer for discriminator (reset at every update of distribution q)
    "batchSize": 500,               # Batch size for discriminator updates
    "n_c_startup": 300,             # Number of steps for discriminator updates at startup
    "n_c": 5,                      # Number of steps for discriminator updates in-between updates of distribution q
    "wdecay": 1e-5,                 # Weight decay factor
    "lrD": 1e-4,                    # Learning rate for discriminator updates
    "lrQ": 1e-4,                    # Learning rate for updates of distribution q
    "tau": 1e-6,                    # Birth-death rate
    "alpha": 0.3,                   # Damping factor ($\alpha$ in Algorithm 3)
    "lambda_aug_init": 0.0,         # Initialization of augmented Lagrange multiplier (in Algorithm 3)
    "rho": 1e-6                     # Learning rate of augmented Lagrange multiplier
}
\end{lstlisting}

\begin{lstlisting}[language=Python, caption={Hyperparameters for single-cell analysis interpolation experiments (Fig.\ \ref{fig:wot}, \ref{fig:wotsupp})}]
{
    "n_layers": [128, 1024, 64],   # Number of neurons in hidden layers of discriminator
    "n_points_src": 3500,           # Number of points sampled from source distribution
    "n_points_target": 3500,        # Number of points sampled from target distribution
    "T": 400,                       # Number descent steps
    "optimizer": Adam(amsgrad=True)     # Optimizer for discriminator (reset at every update of distribution q)
    "batchSize": 100,               # Batch size for discriminator updates
    "n_c_startup": 300,             # Number of steps for discriminator updates at startup
    "n_c": 5,                      # Number of steps for discriminator updates in-between updates of distribution q
    "wdecay": 1e-5,                 # Weight decay factor
    "lrD": 1e-4,                    # Learning rate for discriminator updates
    "lrQ": 1e-4,                    # Learning rate for updates of distribution q
    "tau": 2e-4,                    # Birth-death rate
    "alpha": 0.2,                   # Damping factor ($\alpha$ in Algorithm 3)
    "lambda_aug_init": 1e-5,         # Initialization of augmented Lagrange multiplier (in Algorithm 3)
    "rho": 1e-6                     # Learning rate of augmented Lagrange multiplier
    "normalization": nn.BatchNorm1d(track_running_stats=False, momentum=0.0)  # Substitutes dropout layer after second hidden layer
}
\end{lstlisting}

\section{Additional Plots}\label{app:plots}
\subsection{Synthetic Examples}\label{app:Synthetic}
We give in Figs \ref{fig:Gauss2circles} and \ref{fig:disk2heart} additional synthetic experiments:
\begin{figure*}[ht!]
\begin{subfigure}{.65\textwidth}
  \centering
  \includegraphics[width=\linewidth]{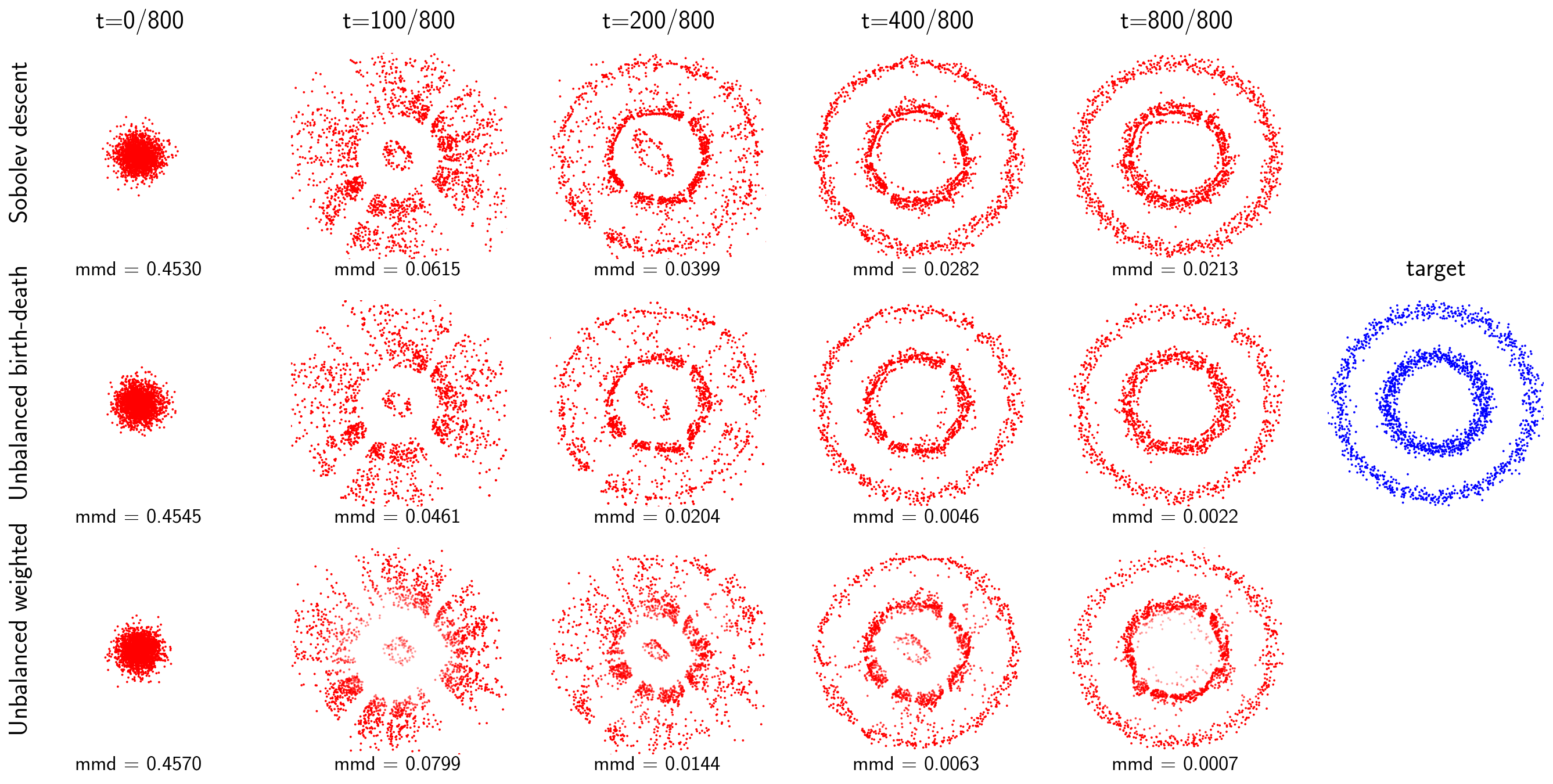}
  \caption{Neural Unbalanced Sobolev Descent paths in transporting a Gaussian to circles). We compare Sobolev descent  (SD, \cite{SD}) to both USD implementations with birth and death processes (bd: Algorithm \ref{alg:NSDDeathBirth}) as well as the weighted version implementation (w: Algoritm \ref{alg:NSD}, note that in this case we overlay the points with their respective weights where coloring density encodes the weights). We see that birth and death processes helps USD to outperform SD in capturing the two modes.}
\end{subfigure}%
\hspace*{0.2in}
\begin{subfigure}{.3\textwidth}
\vspace{0.55in}
\centering
  \includegraphics[width=1.1\linewidth]{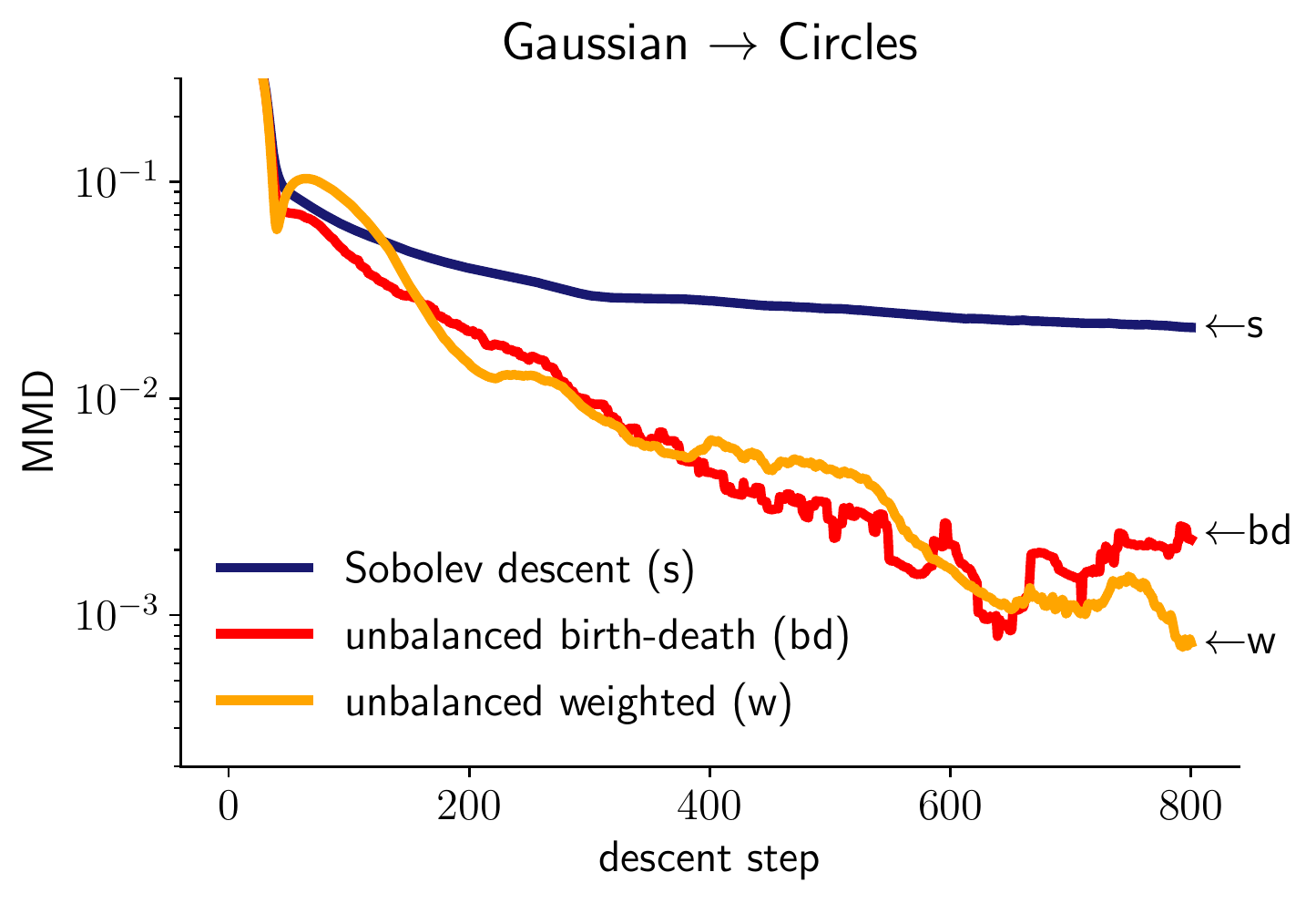}
  \caption{MMD function of the time in the descent from a Gaussian to Circles: We see that birth and death processes in both implementations of USD accelerate the convergence to the target distribution and reaches lower MMD than Sobolev Descent that relies on advection only.}
\end{subfigure}
\caption{Neural Unbalanced Sobolev Descent transporting a Gaussian to circles (target samples have uniform weights, $a_j=\frac{1}{n}$).}
\label{fig:Gauss2circles}
\vskip -0.1in
\end{figure*}

\begin{figure*}[ht!]
\begin{subfigure}{.65\textwidth}
  \centering
  \includegraphics[width=\linewidth]{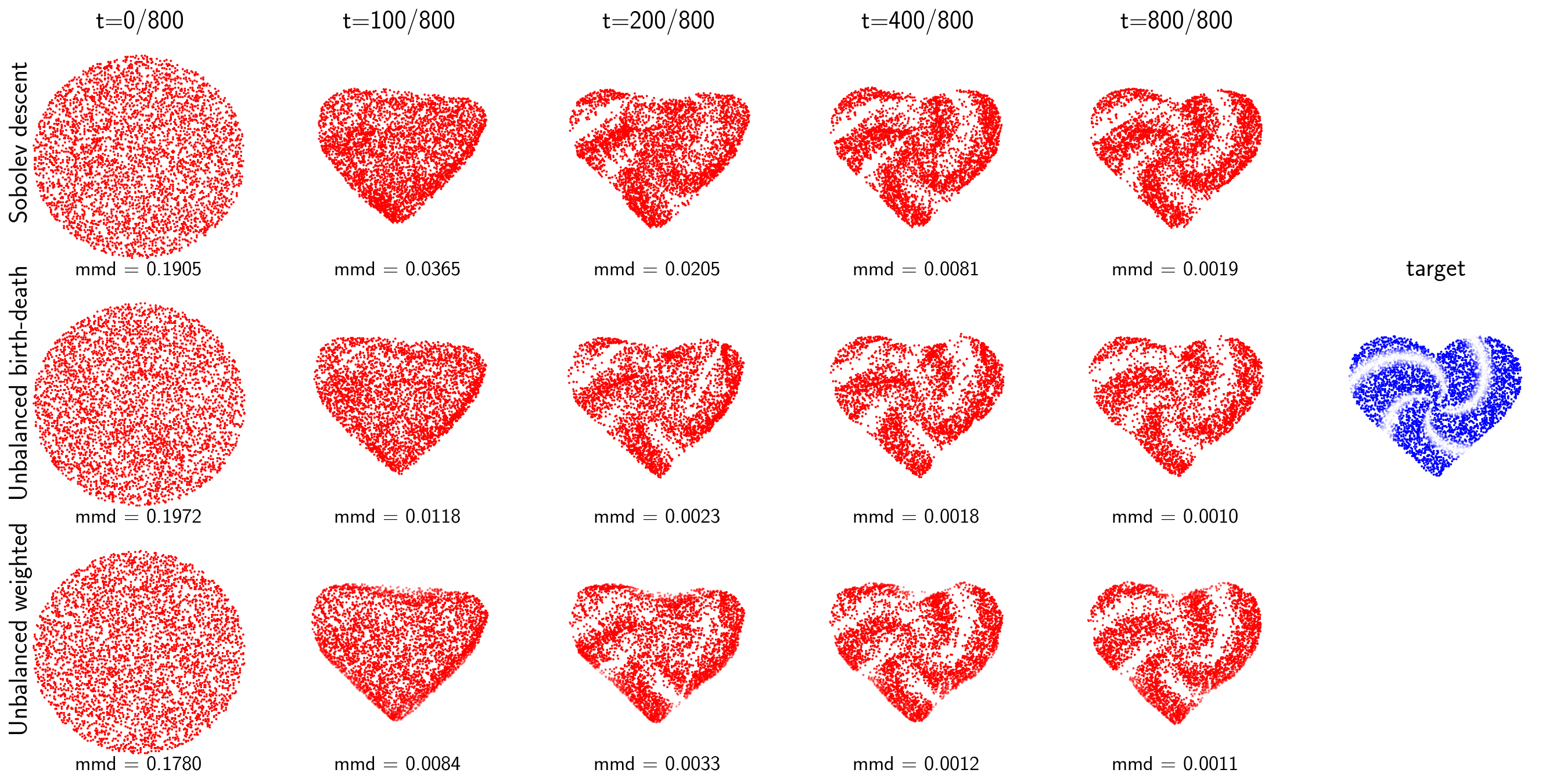}
  \caption{Neural Unbalanced Sobolev Descent paths in transporting a disk to a heart/spiral. We compare Sobolev descent  (SD, \cite{SD}) to both USD implementations with birth and death processes (bd: Algorithm \ref{alg:NSDDeathBirth}) as well as the weighted version implementation (w: Algoritm \ref{alg:NSD}, note that in this case we overlay the points with their respective weights where coloring density encodes the weights). We see that birth and death processes helps USD to outperform SD in capturing the two modes.}
\end{subfigure}%
\hspace*{0.2in}
\begin{subfigure}{.3\textwidth}
\vspace{0.55in}
\centering
  \includegraphics[width=1.1\linewidth]{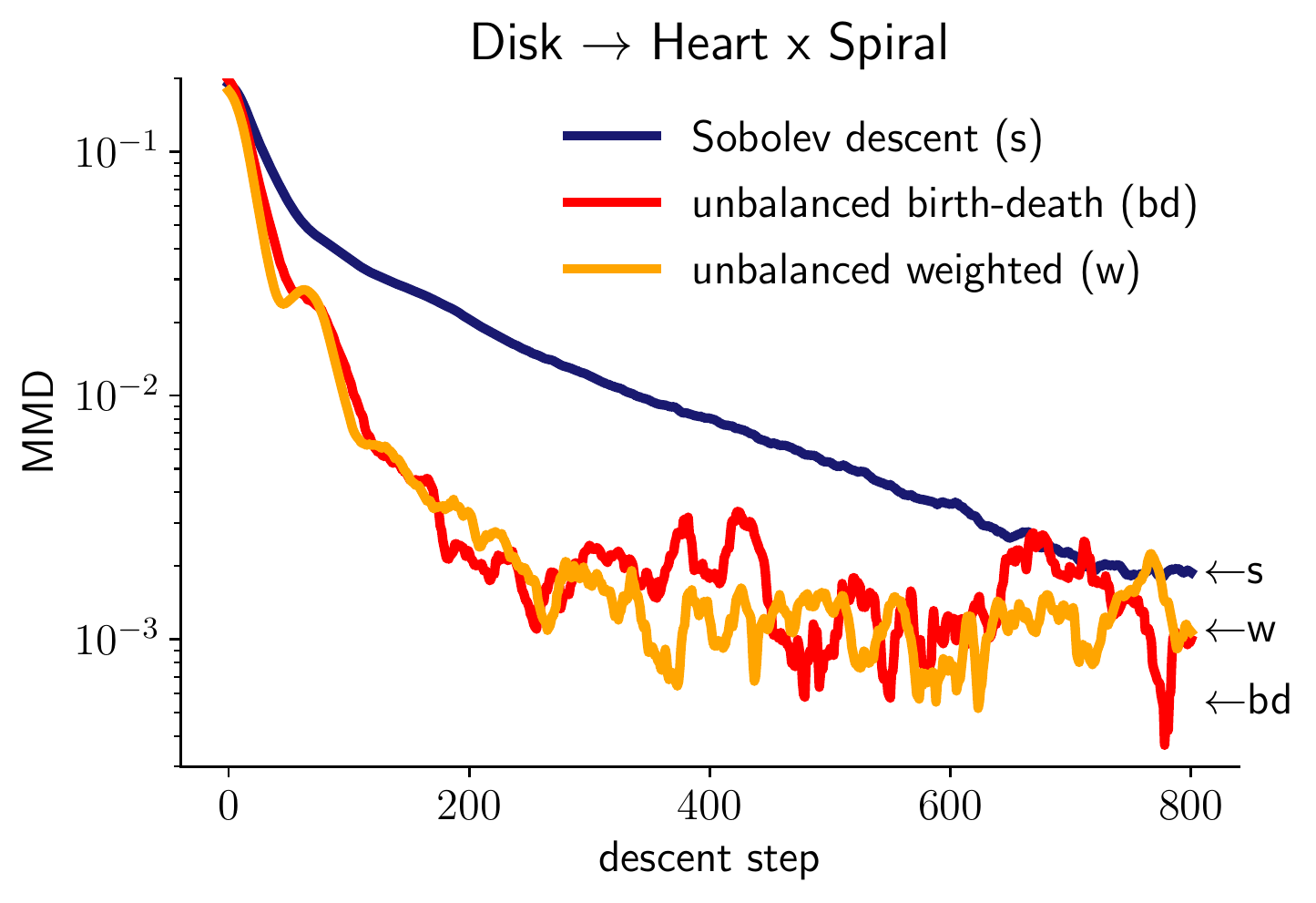}
  \caption{MMD function of the time in the descent from a disk to a heart/spiral: We see that birth and death processes in both implementations of USD accelerate the convergence to the target distribution and reaches lower MMD than Sobolev Descent that relies on advection only.}
\end{subfigure}
\caption{Neural Unbalanced Sobolev Descent transporting a `disk' to a `heart' weighted by a spiral-shaped gradient.}
\label{fig:disk2heart}
\vskip -0.1in
\end{figure*}

\subsection{Image Coloring}\label{app:color}
We give in Figure \ref{fig:imageColorsequence}, the trajectories of the descent in image color transfer experiment. 
\begin{figure*}[ht!]
\begin{subfigure}{\textwidth}
  \centering
  \includegraphics[width=\linewidth]{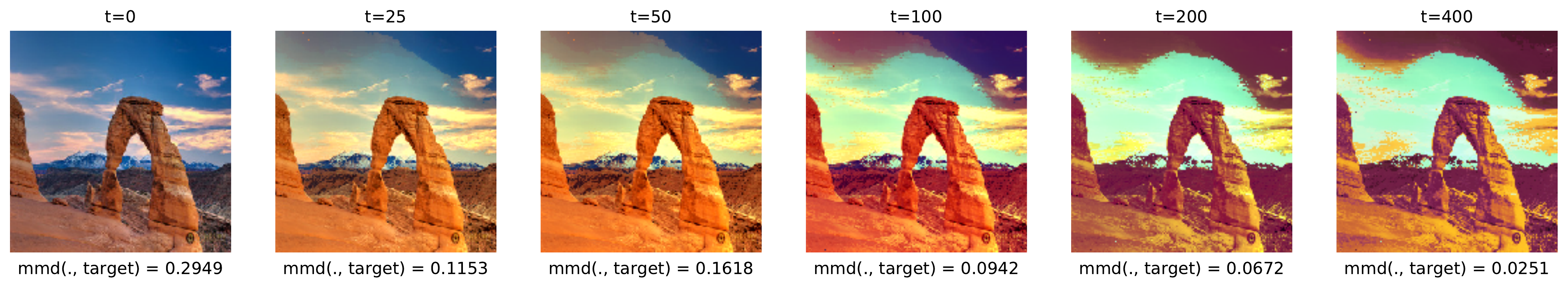}
\end{subfigure}%

\begin{subfigure}{\textwidth}
\centering
  \includegraphics[width=\linewidth]{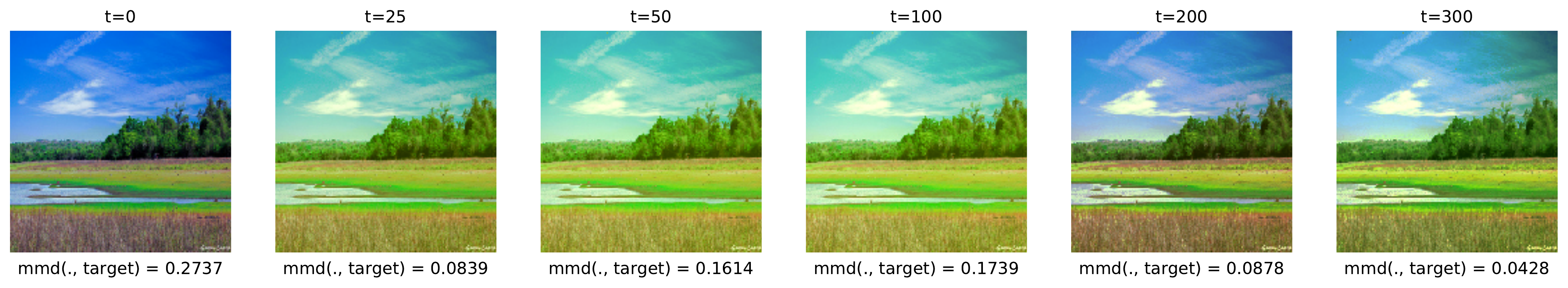}
\end{subfigure}
\caption{ Color Transfer with USD using (bd) Algorithm \ref{alg:NSDDeathBirth}. Trajectories of the descent.}
\label{fig:imageColorsequence}
\end{figure*}

\subsection{Comparisons to Waddington Optimal Transport for single-cell analysis}
We give in Figure \ref{fig:wotsupp} the evolution of the MMD as function of the day of interpolation using USD and unbalanced OT as in WOT.
\begin{figure}[ht!]
    \includegraphics[width=\textwidth]{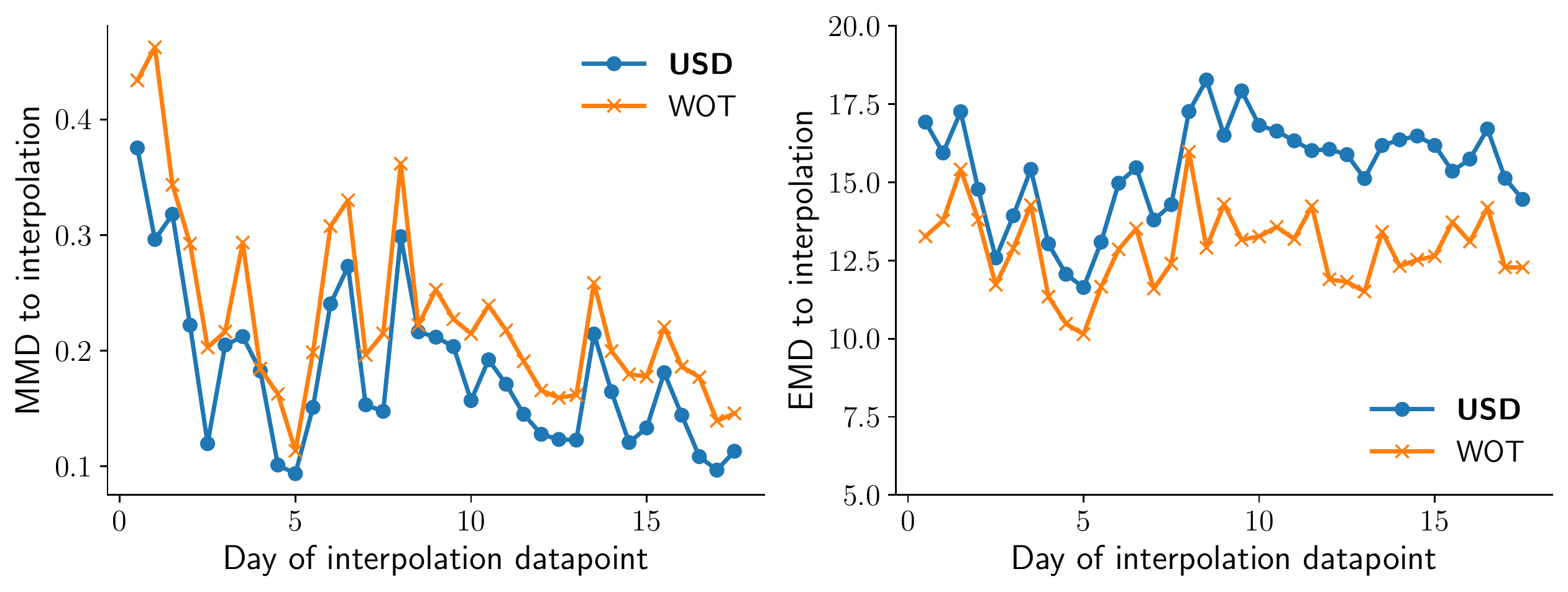}
    \caption{MMD and EMD between predicted mid points (using USD and WOT)  and their respective ground truths as function of the day of interpolation. }
    \label{fig:wotsupp}
\end{figure}

\end{document}